\documentclass[11pt]{article} 
\usepackage{jmlr2e} 

\usepackage{lastpage}

\usepackage{wrapfig}
\usepackage{lipsum}
\usepackage{mwe}
\usepackage{algorithm}
\usepackage{algpseudocode}%

\usepackage{amsmath,amssymb}

\usepackage{microtype}
\usepackage{graphicx}
\usepackage{subfigure}
\usepackage{booktabs} 

\usepackage[table]{xcolor}
\usepackage{epstopdf}
\usepackage{amsmath}
\usepackage{tcolorbox}
\usepackage{array}
\usepackage{hyperref}
\usepackage{amsmath,amssymb}
\usepackage{tikz}
\usepackage{float}

\usepackage[utf8]{inputenc}
\usepackage[english]{babel}
\usepackage{natbib}

\usepackage[utf8]{inputenc} 
\usepackage[T1]{fontenc}    
\usepackage{hyperref}       
\usepackage{url}            
\usepackage{booktabs}       
\usepackage{amsfonts}       
\usepackage{nicefrac}       
\usepackage{microtype}      
 
\usepackage{multicol}
\usepackage{color}
 
\newtheorem{thm}{Theorem}[section]
\newtheorem{lem}{Lemma}[section]

\usepackage{comment}

 \makeatletter
\renewcommand*{\thanks}[1]{%
  \footnotemark
  \protected@xdef\@thanks{\@thanks
    \protect\footnotetext[\arabic{footnote}]{#1}}%
}
\makeatother

 \jmlrheading{21}{2020}{1-\pageref{LastPage}}{6/18; Revised
11/20}{12/20}{18-352}{Rachel Ward, Xiaoxia Wu and  L{\'e}on Bottou}
\ShortHeadings{AdaGrad stepsizes: Sharp convergence over nonconvex landscapes}{Ward, Wu and Bottou}

\firstpageno{1}
 
\begin{document}
\title{AdaGrad stepsizes: Sharp convergence over nonconvex landscapes}

\author{
\\
 \name  Rachel Ward \textbf{$^*$}  \email rward@math.utexas.edu \\ 
 \name Xiaoxia Wu  \thanks{Equal Contribution; work done at Facebook AI Research.} \email xwu@math.utexas.edu \\
 \addr Department of Mathematics\\
  The University of Texas at Austin\\
2515 Speedway, Austin, TX, 78712, USA   
 \AND
 \name  L{\'e}on Bottou \email leonb@fb.com \\
   \addr  Facebook AI Research \\
  \addr    770 Broadway, New York, NY, 10019, USA\\
}
\editor{Mark Schmidt}

 \maketitle

\begin{abstract}
Adaptive gradient methods such as AdaGrad and its variants update the stepsize in stochastic gradient descent on the fly according to the gradients received along the way; such methods have gained widespread use in large-scale optimization for their ability to converge robustly, without the need to fine-tune the stepsize schedule. Yet, the theoretical guarantees to date for AdaGrad are for online and convex optimization.   We bridge this gap by providing theoretical guarantees for the convergence of AdaGrad for smooth, nonconvex functions.  We show that the norm version of AdaGrad (AdaGrad-Norm) converges to a stationary point at the $\mathcal{O}(\log(N)/\sqrt{N})$ rate in the stochastic setting, and at the optimal $\mathcal{O}(1/N)$ rate in the batch (non-stochastic) setting -- in this sense, our convergence guarantees are ``sharp''.  In particular, the convergence of AdaGrad-Norm is robust to the choice of all hyper-parameters of the algorithm, in contrast to stochastic gradient descent whose convergence depends crucially on tuning the step-size to the (generally unknown) Lipschitz smoothness constant and level of stochastic noise on the gradient. Extensive numerical experiments are provided to corroborate our theoretical findings; moreover, the experiments suggest that the robustness of AdaGrad-Norm extends to the models in deep learning.
\end{abstract}


\section{Introduction}
\label{sec:intro}

Consider the problem of minimizing a differentiable non-convex function $F: \mathbb{R}^d \rightarrow \mathbb{R}$ via stochastic gradient descent (SGD); starting from $x_0 \in \mathbb{R}^d$ and stepsize $\eta_0 > 0$, SGD iterates until convergence
\begin{align}
\label{GD_basic}
x_{j+1} &\leftarrow x_j - \eta_j  G(x_j),
\end{align}
where $\eta_j > 0$ is the stepsize at the $j$th iteration and $G(x_j)$ is the stochastic gradient in the form of a random vector satisfying $\mathbb{E}[G(x_j)] = \nabla F(x_j)$ and having bounded variance.  SGD  is the de facto standard for deep learning optimization problems, or more generally, for the large-scale optimization problems where the loss function $F(x)$ can be approximated by the average of a large number $m$ of component functions, $F(x) =\frac{1}{m} \sum_{i=1}^m f_i(x)$. It is more efficient to measure a single component gradient $\nabla f_{i_j}(x),  i_j \sim \text{Uniform}\{1,2,\dots, m\}$ (or subset of component gradients), and move in the noisy direction $G_j(x) = \nabla f_{i_j}(x),$ than to compute a full gradient $\frac{1}{m}\sum_{i=1}^m \nabla f_i(x)$. 

For non-convex but smooth loss functions $F$, (noiseless) gradient descent (GD) with constant stepsize converges to a stationary point of $F$ at rate $\mathcal{O} \left({1}/{N} \right)$ with the number of iterations $N$ \citep{nesterov1998introductory}. In the same setting, and under the general assumption of bounded gradient noise variance, SGD with constant or decreasing stepsize $ \eta_j =\mathcal{O}\left({1}/{\sqrt{j}}\right)$ has been proven to converge to a stationary point of $F$ at rate  $\mathcal{O} \left({1}/{\sqrt{N}} \right)$  \citep{ghadimi2013stochastic, BCN16}.  The $\mathcal{O} \left({1}/{N} \right)$  rate for GD is the best possible worst-case dimension-free rate of convergence for any algorithm \citep{carmon2019lower}; { faster convergence rates in the noiseless setting are available under the mild assumption of additional smoothness \citep{agarwal2017, carmonconvex, carmon2018accelerated}. In the noisy setting, faster rates than $\mathcal{O} \left({1}/{\sqrt{N}} \right)$  are also possible  using accelerated SGD methods \citep{ghadimi2016accelerated, allen2016improved, reddi2016fast, allen2017natasha, xu2017neon, zhou2018stochastic,fang2018nips}.  For instance, \cite{zhou2018stochastic} and \cite{fang2018nips} obtain the rate $\mathcal{O} \left({1}/{{N}^{2/3}} \right)$ without requiring finite-sum
structure but with an additional assumptions about Lipschitz
continuity of the stochastic gradients, which they exploit to reduce variance.}

Instead of focusing on faster convergence rates for SGD, this paper focuses on adaptive stepsizes \citep{cutkosky2017online,levy2017online} that make the optimization algorithm more robust to (generally unknown) parameters of the optimization problem, such as the noise level of the stochastic gradient and the Lipschitz smoothness constant $L$ of the loss function defined as the smallest number $L > 0$ such that $\| \nabla F(x) - \nabla F(y) \| \leq L \| x - y \|$ for all $x ,y$. 
In particular, the $\mathcal{O} \left({1}/{N} \right)$ convergence of GD with fixed stepsize is guaranteed only if the fixed stepsize $\eta > 0$ is carefully chosen such that $\eta \leq 1/L$ -- choosing a larger stepsize $\eta$, even just by a factor of 2, can result in oscillation or divergence of the algorithm \citep{nesterov1998introductory}.  Because of this sensitivity, GD with fixed stepsize is rarely used in practice; instead, one adaptively chooses the stepsize $\eta_j > 0$ at each iteration to approximately maximize a decrease of the loss function in the current direction of $-\nabla F(x_j)$ via either line search \citep{Nocedal2006}, or according to the Barzilai-Borwein rule \citep{BBmethod} combined with line search.   

Unfortunately, in the noisy setting where one uses SGD for optimization, line search methods are not useful, as in this setting the stepsize should not be overfit to the noisy stochastic gradient direction at each iteration. The classical Robbins/Monro theory \citep{rm51} says that in order for $\lim_{k \rightarrow \infty} \mathbb{E}[ \|\nabla F(x_k)\|^2 ] = 0,$ the stepsize schedule should satisfy
\begin{equation}
\label{step:rm}
\sum_{k=1}^{\infty} \eta_k = \infty \quad \text{and} \quad \sum_{k=1}^{\infty} \eta_k^2 < \infty.
\end{equation}
However, these bounds do not tell us much about how to select a good stepsize schedule in practice, where algorithms are run for finite iterations and the constants in the rate of convergence matter. 

The question of how to choose the stepsize $\eta > 0$ or stepsize or learning rate schedule $\{\eta_j \}$ for SGD is by no means resolved; in practice, a preferred schedule is chosen manually by testing many different schedules in advance and choosing the one leading to smallest training or generalization error.   This process can take days or weeks, and can become prohibitively expensive in terms of time and computational resources incurred.

\subsection{Stepsize adaptation with AdaGrad-Norm}
Adaptive stochastic gradient methods such as \emph{AdaGrad} (introduced independently by \citet{duchi2011adaptive} and \citet{mcmahan2010adaptive}) have been widely used in the past few years.  AdaGrad updates the stepsize $\eta_j$ on the fly given information of all previous (noisy) gradients observed along the way.  The most common variant of AdaGrad updates an entire vector of per-coefficient stepsizes \citep{lafond-vasilache-bottou-2017}. To be concrete, for optimizing a function $F: \mathbb{R}^d \rightarrow \mathbb{R}$,  the ``coordinate'' version of AdaGrad updates $d$ scalar parameters $b_j(k), k=1,2,\ldots,d$ at the $j$ iteration --  one for each $x_j(k)$ coordinate of $x_j \in \mathbb{R}^d$ -- according to $b_{j+1}(k)^2 = b_j(k)^2 + [\nabla F(x_j)]_k^2$ in the noiseless setting, and $b_{j+1}(k)^2 = b_j(k)^2 + [G_j(k)]^2$ in the noisy gradient setting.  This  common use makes AdaGrad a variable metric method and has been the object of recent criticism for machine learning applications \citep{wilson2017marginal}.  

One can also consider a variant of AdaGrad which updates only a single (scalar) stepsize according to the sum of squared gradient norms observed so far.  In this work, we focus instead on the ``norm'' version of AdaGrad as a single stepsize adaptation method using the gradient norm information, which we call AdaGrad-Norm. The update in the stochastic setting is as follows: initialize a single scalar $b_0 > 0$; at the $j$th iteration, observe the  random variable $G_j$ such that $\mathbb{E}[G_j] = \nabla F(x_j)$ and iterate
\begin{align*}
x_{j+1} &\leftarrow x_j - \eta \frac{G(x_j)}{b_{j+1}} \quad \text{with}\quad  b_{j+1}^2= b_j ^2+ \| G(x_j) \|^2
\end{align*}
where $\eta > 0$ is to ensure homogeneity and that the units match. It is straightforward that in expectation, $\mathbb{E} [b_k^2] = b_0^2 + \sum_{j=0}^{k-1} \mathbb{E}[ \| G(x_j) \|^2];$ thus, under the assumption of uniformly bounded gradient $\|\nabla F(x)\|^2 \leq \gamma^2$ and uniformly bounded variance  $\mathbb{E}_{\xi}\left[\| G(x; \xi)-\nabla F(x)\|^2\right]\leq \sigma^2$, the stepsize will decay eventually according to $\frac{1}{b_j} \geq \frac{1}{\sqrt{2(\gamma^2+\sigma^2)j}}.$  This stepsize schedule matches the schedule which leads to optimal rates of convergence for SGD in the case of convex but not necessarily smooth functions, as well as smooth but not necessarily convex functions (see, for instance, \citet{agarwal2009information} and \citet {bubeck2015convex}). This observation suggests that AdaGrad-Norm should be able to achieve convergence rates for SGD, but \emph{without having to know Lipschitz smoothness parameter of $F$ and the parameter $\sigma$ a priori} to set the stepsize schedule.  

Theoretically rigorous convergence results for AdaGrad-Norm were provided in the convex setting recently \citep{levy2017online}. Moreover, it is possible to obtain convergence rates in the offline setting by online-batch conversion. However, making such observations rigorous for nonconvex functions is difficult because $b_j$ is itself a random variable which is correlated with the current and all previous noisy gradients; thus, the standard proofs in SGD do not straightforwardly extend to the proofs of AdaGrad-Norm. This paper provides such a proof for AdaGrad-Norm.

 \subsection{Main contributions} 
Our results make rigorous and precise the observed phenomenon that the convergence behavior of AdaGrad-Norm is \emph{highly adaptable to the unknown Lipschitz smoothness constant and level of stochastic noise on the gradient}: when there is noise, AdaGrad-Norm converges at the rate of $O(\log(N)/\sqrt{N})$, and when there is no noise, the same algorithm converges at the optimal $O(1/N)$ rate like well-tuned batch gradient descent.  Moreover, our analysis shows that AdaGrad-Norm converges at these rates for any choices of the algorithm hyperparameters $b_0 > 0$ and $\eta > 0$, in contrast to GD or SGD with fixed stepsize where if the stepsize is set above a hard upper threshold governed by the (generally unknown) smoothness constant $L$, the algorithm might not converge at all.  Finally, we note that the constants in the rates of convergence we provide are explicit in terms of their dependence on the hyperparameters $b_0$ and $\eta$.  We list our two main theorems (informally) in the following: 
 \begin{itemize}
\item
For a differentiable non-convex function $F$ with $L$-Lipschitz gradient  and $F^{*} = \inf_{x}F(x)>-\infty$, Theorem \ref{thm:improveSGDmain} implies that AdaGrad-Norm converges to an $\varepsilon$-approximate stationary point with high probability \footnote{It is becoming  common to define an $\varepsilon$-approximate stationary point as $\|\nabla F(x)\| \leq \varepsilon$ \citep{agarwal2017, carmon2018accelerated, carmon2019lower, fang2018nips, zhou2018stochastic, zhu2018natasha}, but we use the convention $\| F(x) \|^2 \leq \varepsilon$  \citep{lei2017non,  BCN16} to most easily compare our results to those from \citet{ghadimi2013stochastic,orabona18}.} at the rate   
\begin{align*}
     &\min_{\ell \in [N-1]}  \|\nabla F(x_{\ell})\|^2  \leq  \mathcal{O} \left( \frac{ \gamma (\sigma+ \eta L +  (F(x_0)-F^*)/\eta) \log(N \gamma^2/b_0^2) }{ \sqrt{N}} \right).
\end{align*}
If the optimal value of the loss function $F^{*}$ is known and one sets $\eta = F(x_0) - F^{*}$ accordingly, then the constant in our rate is close to the best-known constant $\sigma L (F(x_0)-F^*)$ achievable for SGD with fixed stepsize $\eta = \eta_1 = \dots = \eta_N =  \min \{ \frac{1}{L}, \frac{1}{\sigma \sqrt{N}} \}$ carefully tuned to knowledge of $L$ and $\sigma$, as given in \citet{ghadimi2013stochastic}.  However, our result requires bounded gradient $\| \nabla F(x) \|^2 \leq \gamma^2$ and our rate constant scales with $\gamma \sigma$ instead of linearly in $\sigma$.  Nevertheless, our result suggests a good strategy for setting hyperparameters in implementing AdaGrad-Norm practically: given knowledge of $F^{*}$, set $\eta = F(x_0) - F^{*}$ and simply initialize $b_0 > 0$ to be very small.
\item
When there is no noise $\sigma = 0$, we can improve this rate to an $\mathcal{O}\left(1/N\right)$ rate of convergence.  In Theorem \ref{thm:converge}, we show that   $\min_{j\in[N]} \| \nabla F(x_j) \|^2 \leq \varepsilon$
after \\
(1) $N =  \mathcal{O}\left( \frac{1}{\varepsilon} \left( \left({(F(x_0)- F^*)}/{\eta}\right)^2 +b_0 \left(F(x_0)- F^* \right)/\eta \right) \right)$  if $ b_0 \geq \eta  L,$ \\
(2) $N =  \mathcal{O}\left(\frac{1}{\varepsilon} \left(
L\left(F(x_{0})- F^* \right) + \left({(F(x_0)- F^*)}/{\eta}\right)^2 \right) + \frac{(\eta L)^{2}}{\varepsilon}\log\left(\frac{\eta L}{b_0}\right)\right)$  if ${b_0}< {\eta} L.$  \\
Note that the constant $(\eta L)^2$ in the second case when $ {b_0} < \eta L$  is not optimal compared to the known best rate constant $\eta L$ obtainable by gradient descent with fixed stepsize $\eta = 1/L$  \citep{carmon2019lower}; on the other hand, given knowledge of $L$ and $F(x_0) - F^{*}$, the rate constant of AdaGrad-norm reproduces the optimal constant $\eta L$ by setting $\eta = F(x_0)- F^*$ and $b_0 = \eta L$.  
 \end{itemize}
Practically, our results imply a good strategy for setting the hyperparameters when implementing AdaGrad-norm in practice: set $\eta = (F(x_0)- F^*)$ (assuming $F^{*}$ is known) and set  $b_0 > 0$ to be a very small value.  If $F^{*}$ is unknown, then setting $\eta = 1$ should work well for a wide range of values of $L$, and in the noisy case with $\sigma^2$ strictly greater than zero.

\subsection{Previous work}
Theoretical guarantees of convergence for AdaGrad were provided in \citet{duchi2011adaptive} in the setting of online convex optimization, where the loss function may change from iteration to iteration and be chosen adversarially.   AdaGrad was subsequently observed to be effective for accelerating convergence in the nonconvex setting, and has become a popular algorithm for optimization in deep learning problems.  Many modifications of AdaGrad with or without momentum have been proposed, namely,  RMSprop \citep{hinton2012neural}, AdaDelta \citep{zeiler2012adadelta}, Adam \citep{kingma2014adam}, AdaFTRL\citep{OP15}, SGD-BB\citep{tan2016barzilai}, AdaBatch \citep{defossez2017adabatch},  SC-Adagrad \citep{pmlr-v70-mukkamala17a}, AMSGRAD \citep{j.2018on}, Padam \citep{chen2018closing}, etc. Extending our convergence analysis to these popular alternative adaptive gradient methods remains an interesting problem for future research.

Regarding the convergence guarantees for the norm version of adaptive gradient methods in the offline setting, the recent work by \citet{levy2017online} introduces a family of adaptive gradient methods inspired by AdaGrad, and proves convergence rates in the setting of (strongly) convex loss functions without knowing the smoothness parameter $L$  in advance. Yet, that analysis still requires the a priori knowledge of a convex set ${\cal K}$ with known diameter $D$ in which the global minimizer resides.  
 More recently,  \citet{wwb18} provids convergence guarantees in the non-convex setting for a different adaptive gradient algorithm, WNGrad, which is closely related to AdaGrad-Norm and inspired by weight normalization \citep{salimans2016weight}.   In fact, the WNGrad stepsize update is similar to AdaGrad-Norm's: 
\begin{align*}
\text{(WNGrad)} \quad b_{j+1} &= b_j +\| \nabla F(x_j) \|/b_j;\\
\text{(AdaGrad-Norm)} \quad b_{j+1} &= b_j +\| \nabla F(x_j) \|/(b_j+b_{j+1}).
\end{align*}
However, the  guaranteed convergence in \citet{wwb18}  is only for the batch setting and the constant in the convergence rate is worse than the one provided here for AdaGrad-Norm. Independently, \citet{orabona18} also proves  the $O(1/\sqrt{N})$ convergence rate for a variant of   AdaGrad-Norm in the non-convex stochastic setting, but their analysis requires knowledge of of smoothness constant $L$ and a hard threshold of $b_0 >\eta L$  for their convergence.   In contrast to \citet{orabona18}, we do not require knowledge of the Lipschitz smoothness constant $L$, but we do assume that the gradient $\nabla F$ is uniformly bounded by some (unknown) finite value, while \cite{orabona18} only assumes bounded variance $\mathbb{E}_{\xi}\left[\| G(x; \xi)-\nabla F(x)\|^2\right]\leq \sigma^2.$

\subsection{Future work}
This paper provides convergence guarantees for AdaGrad-Norm over smooth, nonconvex functions, in both the stochastic and deterministic settings.  Our theorems should shed light on the popularity of AdaGrad as a method for more robust convergence of SGD in nonconvex optimization in that the convergence guarantees we provide are robust to the initial stepsize $\eta /b_0$, and adjust automatically to the level of stochastic noise.   Moreover, our results suggest a good strategy for setting hyperparameters in AdaGrad-Norm implementation: set $\eta = (F(x_0)- F^*)$ (if $F^{*}$ is known) and set  $b_0 > 0$ to be a very small value. However, several improvements and extensions should be possible.   First, the constant in the convergence rate we present can likely be improved and it remains open whether we can remove the assumption of the uniformly bounded gradient in the stochastic setting. It would be interesting to analyze AdaGrad in its coordinate form, where each coordinate $x(k)$ of $x \in \mathbb{R}^d$ has its own stepsize $\frac{1}{b_j(k)}$ which is updated according to $b_{j+1}(k)^2 = b_j(k)^2 + [\nabla F(x_j)]_k^2$. AdaGrad is just one particular adaptive stepsize method and other updates such as Adam \citep{kingma2014adam} are often preferable in practice; it would be nice to have similar theorems for other adaptive gradient methods, and to even use the theory as a guide for determining the ``best'' method for adapting the stepsize for given problem classes.

\subsection{Notation}
Throughout, $\| \cdot \|$ denotes the $\ell_2$ norm.  We use the notation $[N] := \{ 0,1,2, \dots, N\}$.  A function $F:\mathbb{R}^d \rightarrow \mathbb{R}$ has $L$-Lipschitz smooth gradient if
\begin{equation}
\| \nabla F(x) - \nabla F(y) \| \leq L \| x - y \|, \quad \forall x, y \in \mathbb{R}^d
\label{L-smooth}
\end{equation}
We write $F \in \mathbb{C}_L^1$ and refer to $L$ as the smoothness constant for $F$ if $L > 0$ is the smallest number such that the above is satisfied.

\section{AdaGrad-Norm convergence }
\label{sec:cov}
To be clear about the adaptive algorithm, we first state in Algorithm \ref{alg:Adagrad} the norm version of  AdaGrad we consider throughout in the analysis. 
\begin{algorithm}[H]
  \caption{AdaGrad-Norm} \label{alg:Adagrad}
\begin{algorithmic}[1]
  \State {\bfseries Input:} 
   Initialize $x_0 \in \mathbb{R}^d, b_{0}>0, \eta>0$
    \For{\texttt{ $j = 1,2, \ldots$ }}
     \State Generate  $\xi_{j-1}$ and $G_{j-1} = G(x_{j-1}, \xi_{j-1})$
      \State${b}_{j}^2 \leftarrow  {b}_{j-1} ^2+ { \| G_{j-1} \|^2}$  
       \State $x_{j} \leftarrow x_{j-1} -  \frac{\eta}{{b}_{j}}G_{j-1} $
       \EndFor
\end{algorithmic}
  \end{algorithm}

At the $k$th iteration, we observe a \emph{stochastic gradient}  $G(x_k, \xi_k),$ where $\xi_k$, $k = 0,1,2 \dots$ are random variables, and such that   $G(x_k, \xi_k)$ is an unbiased estimator of $\nabla F(x_k)$.\footnote{$\mathbb{E}_{\xi_k} \left[G(x_k, \xi_k) \right]= \nabla F(x_k)$ where $\mathbb{E}_{\xi_k} \left[\cdot \right]$ is the expectation with respect $\xi_k$ conditional on previous $\xi_0, \xi_1,\ldots, \xi_{k-1}$}
We require the following additional assumptions: for each $k \geq 0$,
\begin{enumerate}
\item The random vectors $\xi_k$, $k = 0,1,2, \dots, $ are independent of each other and also of $x_k$;
\item $\mathbb{E}_{\xi_k}[ \| G(x_k, \xi_k) -\nabla F(x_k) \|^2] \leq \sigma^2;$
\item $\| \nabla F(x) \|^2 \leq \gamma^2$ uniformly.
\end{enumerate}
The first two assumptions are standard (see e.g. \citet{ NY09, NJG83, BCN16}).  The third assumption is somewhat restrictive as it rules out strongly convex objectives, but is not an unreasonable assumption for AdaGrad-Norm, where the adaptive learning rate is a cumulative sum of all previous observed gradient norms.

Because of the variance in gradient, the AdaGrad-Norm stepsize $\frac{\eta}{b_k}$ decreases to zero roughly at a rate between $ \frac{1}{ \sqrt{2(\gamma^2+\sigma^2)k}}$ and $ \frac{1}{ \sigma\sqrt{k}}$.  It is known that AdaGrad-Norm stepsize decreases at this rate \citep{levy2017online}, and that this rate is optimal in $k$ in terms of the resulting convergence theorems in the setting of smooth but not necessarily convex $F$, or convex but not necessarily strongly convex or smooth $F$.  Still, standard convergence theorems for SGD do not extend straightforwardly to AdaGrad-Norm because the stepsize $1/b_k$ is a random variable and dependent on all previous points visited along the way, i.e., $\{\|\nabla F(x_j)\|\}_{j=0}^{k}$ and $\{\|\nabla G(x_j,\xi_{j})\|\}_{j=0}^{k}$.  From this point on, we use the shorthand $G_k = G(x_k, \xi_k)$ and $F_k = \nabla F(x_k)$ for simplicity of notation.  
The following theorem gives the convergence guarantee to Algorithm 1. We give detailed proof in Section \ref{proof_whole}.

\begin{thm}[AdaGrad-Norm: convergence in stochastic setting]
\label{thm:improveSGDmain}
Suppose $F \in \mathbb{C}_L^1$ and $F^{*} = \inf_{x}F(x)>-\infty$.  Suppose that the random variables $G_{\ell}, \ell \geq 0$, satisfy the above assumptions.  Then  with probability $1- \delta$,
\begin{align*}
      &\min_{\ell \in [N-1]}  \|\nabla F(x_{\ell})\|^2\leq \min\bigg\{   \left( \frac{2b_0}{N} + \frac{4(\gamma+\sigma)}{\sqrt{N}}\right)\frac{{\cal Q}}{ \delta^{3/2}},   \left( \frac{8{\cal Q}}{\delta} +  2b_0 \right)\frac{4{\cal Q}}{ N \delta }+ \frac{8{\cal Q}\sigma}{ \delta^{3/2} \sqrt{N}}  \bigg\}
\end{align*}
where    
$${\cal Q} = \frac{ F(x_{0}) - F^{*} }{\eta}+ \frac{4\sigma+\eta L}{2}\log \left(\frac{20N(\gamma^2+\sigma^2)}{b_0^2}+10\right) 
. $$
\end{thm}

This result implies that AdaGrad-Norm converges for any $\eta > 0$ and starting from any value of $b_0 > 0$.  To put this result in context, we can compare to Corollary 2.2 of \citet{ghadimi2013stochastic} giving the best-known convergence rate for SGD with fixed step-size in the same setting (albeit not requiring Assumption (3) of uniformly bounded gradient):  if the Lipschitz  smoothness constant $L$ and the variance $\sigma^2$ are known a priori, and the fixed stepsize in SGD is set to 
$$
\eta= \min \left \{ \frac{1}{L}, \frac{1}{\sigma \sqrt{N}} \right \}, \quad j = 0,1,\dots, N-1, 
$$
then with probability $1-\delta$
$$
\min_{\ell \in [N-1]} \|\nabla F(x_{\ell})\|^2 \leq \frac{2L (F(x_0) - F^{*})}{ N\delta} +  \frac{(L + 2(F(x_0) - F^{*}))\sigma}{\delta\sqrt{N}}.
$$
We match the $O(1/\sqrt{N})$ rate of \citet{ghadimi2013stochastic}, but without a priori knowledge of $L$ and $\sigma$, and with a worse constant in the rate of convergence.  In particular, the constant in our bound  {scales according to $\sigma^2$} or $\gamma \sigma$ (up to logarithmic factors in $\sigma$) while the result for SGD with well-tuned fixed step-size scales linearly with $\sigma$.   The additional logarithmic factor  (by Lemma \ref{lem:logsum}) results from the AdaGrad-Norm update using the square norm of the gradient (see inequality \eqref{eq:logsum} for details).  The extra constant $\frac{1}{\sqrt{\delta}}$ results from the correlation between the stepsize $b_j$ and the gradient $\|\nabla F(x_j)\|$. We note that the recent work \cite{orabona18} derives an $O(1/\sqrt{N})$ rate for a variation of AdaGrad-Norm without the assumption of uniformly bounded gradient,  but at the same time requires a priori knowledge of the smoothness constant $L > 0$ in setting the step-size in order to establish convergence, similar to SGD with fixed stepsize.   Finally, we note that recent works  \citep{allen2017natasha,lei2017non,fang2018nips,zhou2018stochastic} provide modified SGD algorithms with convergence rates faster than $O(1/\sqrt{N})$, albeit again requiring priori knowledge of both $L$ and $\sigma$ to establish convergence.

We reiterate however that the main emphasis in Theorem \ref{thm:improveSGDmain} is on the robustness of the AdaGrad-Norm convergence to its hyperparameters $\eta$ and $b_0$, compared to plain SGD's dependence on its parameters $\eta$ and $\sigma$.  Although the constant in the rate of our theorem is not as good as the best-known constant for stochastic gradient descent with well-tuned fixed stepsize, our result suggests that implementing AdaGrad-Norm allows one to vastly reduce the need to perform laborious experiments to find a stepsize schedule with reasonable convergence when implementing SGD in practice.

We note that for the second bound in \ref{thm:improveSGDmain}, in the limit as $\sigma \rightarrow 0$ we recover an $O\left(\log(N)/N\right)$ rate of convergence for noiseless gradient descent.  We can establish a stronger result in the noiseless setting using a different method of proof, removing the additional log factor and Assumption 3 of uniformly bounded gradient.   We state the theorem below and defer our proof to Section \ref{sec:batch}. 
\begin{thm}[AdaGrad-Norm: convergence in deterministic setting]
\label{thm:converge}
Suppose that $F \in \mathbb{C}_L^1$ and  that $F^{*}= \inf_{x} F(x) > -\infty.$  Consider AdaGrad-Norm in deterministic setting with following update,
\begin{align*}
x_{j} = x_{j-1} -  \frac{\eta}{ {b}_{j}}\nabla F(x_{j-1}) \quad  \text{ with } \quad {b}_{j}^2 = {b}_{j-1} ^2+ { \| \nabla F(x_{j-1}) \|^2} 
\end{align*}
Then 
$
\min_{j\in[N]} \| \nabla F(x_j) \|^2 \leq \varepsilon
$ after
\begin{itemize}
\item[(1)]$N = 1+\lceil { \frac{1}{\varepsilon} \left( \frac{4 \left(F(x_0)- F^* \right)^2}{\eta^2} +\frac{2 b_0\left(F(x_0)- F^* \right)}{\eta} \right) \rceil}$  if $ {b_0}\geq \eta L,$ 
\item[(2)]$N =1+\lceil {\frac{1}{\varepsilon}  \left(2L \left( F(x_{0})- F^* \right) + \left( \frac{2 \left( F(x_{0})- F^* \right)}{\eta}+ \eta L C_{b_0} \right)^2+ (\eta L)^2 ( 1+C_{b_0} ) -b_0^2\right)
\rceil}$ \\ $\quad \quad \text{if }{b_0} <  \eta L.$ Here $C_{b_0} =  1+ 2\log \left(\frac{ \eta L}{b_0} \right)$.
\end{itemize}
 
\end{thm}
The convergence bound shows that, unlike gradient descent with constant stepsize $\eta$ which can diverge if the stepsize $\eta \geq 2/L$, AdaGrad-Norm convergence holds for any choice of parameters $b_0$ and $\eta$.   The critical observation is that if the initial stepsize $\frac{\eta}{b_0}> \frac{1}{L}$ is too large,  the algorithm has the freedom to diverge initially, until $b_j$ grows to a critical point (not too much larger than $L \eta$) at which point $\frac{\eta}{b_j}$ is sufficiently small that the smoothness of $F$ forces $b_j$ to converge to a finite number on the order of $L$, so that the algorithm converges at an $O(1/N)$ rate. To describe the result in Theorem~\ref{thm:converge}, let us first review a classical result (see, for example \cite{nesterov1998introductory}, $(1.2.13)$) on the convergence rate for gradient descent with fixed stepsize.
\begin{lem}
\label{lem:classical}
Suppose that $F \in \mathbb{C}_L^1$ and that $F^{*}= \inf_{x} F(x) > -\infty$.  
Consider gradient descent with constant stepsize, $x_{j+1} = x_j - \frac{\nabla F(x_j)}{b}$.  
If $b \geq L$, then 
$
\min_{j\in [N-1]} \| \nabla F(x_j) \|^2 \leq \varepsilon
$
after at most a number of steps 
$$
N =\frac{2 b (F(x_0)-F^{*})}{\varepsilon}.
$$
Alternatively, if $b \leq \frac{L}{2}$, then convergence is not guaranteed at all -- gradient descent can oscillate or diverge.
\end{lem}
Compared to the convergence rate of gradient descent with fixed stepsize, AdaGrad-Norm in the case $b = b_0 \geq \eta L$ gives a larger constant in the rate.  But in case $b = b_0 <  \eta L$, gradient descent can fail to converge as soon as $b \leq  \eta L/2$, while AdaGrad-Norm converges for any $b_0 > 0$, and is extremely robust to the choice of $b_0 < \eta L$ in the sense that the resulting convergence rate remains close to the optimal rate of gradient descent with fixed stepsize $1/b = 1/  L$, paying a factor of $\log(\frac{ \eta L}{b_0})$ and $(\eta L)^2$ in the constant. Here, the constant  $(\eta L)^2$ results from the worst-cast analysis using Lemma \ref{lem:increase}, which assumes that the gradient $\| \nabla F(x_j)\|^2\approx \varepsilon$ for all $j= 0,1,\ldots $, when in reality the gradient should be much larger at first.  We believe the number of iterations can be improved by a refined analysis, or by considering the setting where $x_0$ is drawn from an appropriate random distribution.

\section{Proof of Theorem \ref{thm:improveSGDmain}}
\label{sec:stochastic}

\label{proof_whole}
We first introduce some important lemmas in subsection \ref{Ingred} and give the main proof of Theorem 2.1 in Subsection 3.2.
\subsection{Ingredients}
\label{Ingred}
We first introduce several lemmas that are used in the proof for  Theorem \ref{thm:improveSGDmain}. 
We repeatedly appeal to the following classical Descent Lemma, which is also the main ingredient in \citet{ghadimi2013stochastic}, and can be proved by considering the Taylor expansion of $F$ around $y$.
\begin{lem}[Descent Lemma]
\label{lem:descend}
Let $F\in C_L^1$.  Then,
$$
F(x) \leq F(y) + \langle{\nabla F(y), x-y \rangle} + \frac{L}{2} \| x - y \|^2.
$$
\end{lem}
We will also use the following lemmas concerning sums of non-negative sequences. 

\begin{lem}
\label{lem:logsum}
 For any non-negative $a_1,\cdots, a_T$, and $a_1 \geq 1$, we have
\begin{equation} 
\label{eq:rachel}
\sum_{\ell=1}^T \frac{  a_{\ell}}{ {\sum_{i=1}^{\ell}a_{i}} }\leq  \log \left({\sum_{i=1}^{T}a_{i}}\right)+1.
\end{equation}
\end{lem}

\begin{proof}
The lemma can be proved by induction.
 That the sum should be proportional to $ \log \left({\sum_{i=1}^{T}a_{i}}\right)$ can be seen by associating to the sequence a continuous function $g: \mathbb{R}^{+} \rightarrow \mathbb{R}$ satisfying $g(\ell) = a_{\ell}, 1 \leq \ell \leq T$, and $g(t) = 0$ for $t \geq T$, and replacing sums with integrals.     
\end{proof}

\subsection{Main proof}
\label{proof}
\begin{proof} For simplicity, we write $F_j = F(x_j)$ and $\nabla F_j = \nabla F(x_j)$.  By Lemma \ref{lem:descend}, for $ j\geq 0,$
\begin{align}
\frac{F_{j+1} - F_j }{\eta}&\leq- \langle{ \nabla F_j,\frac{G_j}{b_{j+1}} \rangle}+\frac{ \eta  L}{2b^2_{j+1}} \|G_j\|^2  \nonumber\\
&=-\frac{  \|\nabla F_{j}\|^2}{b_{j+1}} +\frac{ \langle{ \nabla F_{j}, \nabla F_j - G_j \rangle} }{b_{j+1}}+\frac{ \eta  L\|G_{j}\|^2}{2b^2_{j+1}}.  \nonumber
\end{align}
At this point, we cannot apply the standard method of proof for SGD, since $b_{j+1}$ and $G_j$ are correlated random variables and thus, in particular, for the conditional expectation
\begin{align*}
\mathbb{E}_{\xi_j} \left[ \frac{\langle{ \nabla F_{j}, \nabla F_j - G_j \rangle} }{b_{j+1}}\right] & \neq  \frac{\mathbb{E}_{\xi_j} \left[\langle{ \nabla F_{j}, \nabla F_j - G_j \rangle} \right]}{b_{j+1}} = \frac{1}{b_{j+1}} \cdot 0;
\end{align*}
If we had a closed form expression for $\mathbb{E}_{\xi_j} [\frac{1}{b_{j+1}}]$, we would proceed by bounding this term as
\begin{align}
 \left| \mathbb{E}_{\xi_j} \left[ \frac{1}{b_{j+1}}\langle{ \nabla F_{j}, \nabla F_j - G_j \rangle}   \right] \right| 
 = &\left| \mathbb{E}_{\xi_j} \left[ \left( \frac{1}{b_{j+1}} - \mathbb{E}_{\xi_j} \left[\frac{1}{b_{j+1}}\right] \right) \langle{ \nabla F_{j}, \nabla F_j - G_j \rangle}   \right] \right| \nonumber \\
\leq& \mathbb{E}_{\xi_j}  \left[ \left|  \frac{1}{b_{j+1}} - \mathbb{E}_{\xi_j} \left[ \frac{1}{b_{j+1}} \right] \right | \| \langle{ \nabla F_{j}, \nabla F_j - G_j \rangle} \|  \right].
\end{align}
However, we do not have a closed form expression for $\mathbb{E}_{\xi_j} [\frac{1}{b_{j+1}}]$. We use the estimate $\frac{1}{ \sqrt{b_{j}^2+ \|\nabla F_j\|^2+\sigma^2}}$ as a surrogate for
 $\mathbb{E}_{\xi_j} [\frac{1}{b_{j+1}}]$ to proceed as we have by Jensen inequality  that {\[ \mathbb{E}_{\xi_j} \left[\frac{1}{b_{j+1}}\right]\geq \frac{1}{\mathbb{E}_{\xi_j} \left[b_{j+1}\right]}= \frac{1}{\mathbb{E}_{\xi_j} \left[\sqrt{b^2_{j} + \|G_j\|^2}\right]} \geq  \frac{1}{\sqrt{\mathbb{E}_{\xi_j} \left[b^2_{j} + \|G_j\|^2\right]}}. \]}
Here we will use $\|\nabla F_j\|^2+\sigma^2$ to approximate $\mathbb{E}_{\xi_j}[\|G_j\|^2]$ . 
Condition on $\xi_1, \dots, \xi_{j-1}$ and take expectation with respect to $\xi_{j}$, 
\begin{align*} 
 0 =  \frac{ \mathbb{E}_{\xi_j} \left[ \langle{ \nabla F_j,  \nabla F_j - G_j \rangle} \right]}{ \sqrt{b_{j}^2+ \|\nabla F_j\|^2+\sigma^2} }  = \mathbb{E}_{\xi_j} \left[  \frac{ \langle{ \nabla F_j,  \nabla F_j - G_j \rangle} }{ \sqrt{b_{j}^2+ \|\nabla F_j\|^2+\sigma^2}}\right] 
\end{align*}
 thus,
\begin{align}
&\frac{\mathbb{E}_{\xi_j} [F_{j+1}]-F_j }{\eta}   \nonumber \\
 \leq&\mathbb{E}_{\xi_j} \left[\frac{ \langle{\nabla F_j,  \nabla F_j - G_j \rangle} }{b_{j+1} }-\frac{\langle{\nabla F_j,  \nabla F_j - G_j \rangle} }{ \sqrt{b_{j}^2+ \|\nabla F_j\|^2+\sigma^2}}  \right]- \mathbb{E}_{\xi_j} \left[\frac{ \|\nabla F_{j}\|^2}{b_{j+1}} \right]+ \mathbb{E}_{\xi_j} \left[ \frac{L\eta \|G_j\|^2}{2b^2_{j+1}} \right]   \nonumber \\
=&  \mathbb{E}_{\xi_j} \left[  \left( \frac{1}{ \sqrt{b_{j}^2+ \|\nabla F_j\|^2+\sigma^2}} - \frac{ 1}{b_{j+1} }\right)  \langle{ \nabla F_j, G_j \rangle}\right]
 - \frac{\|\nabla F_{j}\|^2 }{ \sqrt{b_{j}^2+ \|\nabla F_j\|^2+\sigma^2}} + \frac{\eta L}{2} \mathbb{E}_{\xi_j} \left[ \frac{\|G_j\|^2}{b^2_{j+1}} \right]
\label{eq:first}
\end{align}
Now, observe the term
\begin{align*}
 \frac{1}{\sqrt{b_{j}^2+ \|\nabla F_j\|^2+\sigma^2}}- \frac{1}{b_{j+1}} 
=&  \frac{ (\|G_j\|- \|\nabla F_j\|)( \|G_j\|+\|\nabla F_j\|)-\sigma^2}{b_{j+1}\sqrt{b_{j}^2+ \|\nabla F_j\|^2+\sigma^2} \left( \sqrt{b_{j}^2+ \|\nabla F_j\|^2+\sigma^2}+b_{j+1}\right)} \\
\leq& \frac{  \left|  \|G_j\|- \|\nabla F_j\|\right|}{b_{j+1}\sqrt{b_{j}^2+ \|\nabla F_j\|^2+\sigma^2} }+ \frac{\sigma}{ b_{j+1}\sqrt{b_{j}^2+ \|\nabla F_j\|^2+\sigma^2} }
   \end{align*}
thus, applying Cauchy-Schwarz,
\begin{align}
&\mathbb{E} _{\xi_j} \left[  \left(\frac{1}{\sqrt{b_{j}^2+ \|\nabla F_j\|^2+\sigma^2}} -  \frac{1}{b_{j+1}} \right) \langle{ \nabla F_{j}, G_{j}\rangle}  \right] \nonumber\\
 \leq & \mathbb{E} _{\xi_j} \left[  \frac{ \left| \|G_j\|- \|\nabla F_j\| \right| \|G_j\|\| \nabla F_{j}\|}{b_{j+1}\sqrt{b_{j}^2+ \|\nabla F_j\|^2+\sigma^2} } \right] +  \mathbb{E} _{\xi_j} \left[\frac{\sigma \|G_j\|\| \nabla F_{j}\|}{ b_{j+1}\sqrt{b_{j}^2+ \|\nabla F_j\|^2+\sigma^2} }  \right]\label{eq:keyinq1}
\end{align}
By applying the inequality  $ a b \leq  \frac{1}{2\lambda} b^2+\frac{\lambda }{2}a^2 $  with $\lambda=\frac{2\sigma^2}{\sqrt{b_{j}^2+ \|\nabla F_j\|^2+\sigma^2 }}$, $a=\frac{\|G_j\|}{b_{j+1}}$, and {$b= \frac{\left| \|G_j\|- \|\nabla F_j\| \right| \| \nabla F_{j}\| }{ \sqrt{b_{j}^2+ \|\nabla F_j\|^2+\sigma^2}}$}, the first term in \eqref{eq:keyinq1} can be bounded as\begin{align}
&\mathbb{E}_{\xi_j} \left[  \frac{ \left| \|G_j\|- \|\nabla F_j\| \right| \|G_j\|\| \nabla F_{j}\|}{b_{j+1}\sqrt{b_{j}^2+ \|\nabla F_j\|^2+\sigma^2} } \right] \nonumber\\
\leq   &   \frac{\sqrt{b_{j}^2+ \|\nabla F_j\|^2+\sigma^2}}{4\sigma^2}\frac{  \|\nabla F_{j}\|^2 \mathbb{E}_{\xi_{j}} \left[ \left(\|G_j\|- \|\nabla F_j\|\right)^2 \right]}{b_{j}^2+\|\nabla F_{j}\|^2+\sigma^2} 
+\frac{\sigma^2}{\sqrt{b_{j}^2+ \|\nabla F_j\|^2+\sigma^2}}\mathbb{E}_{\xi_j} \left[ \frac{\|G_{j}\|^2}{b_{j+1}^2} \right] \nonumber\\
\leq    &  \frac{ \|\nabla F_{j}\|^2}{4\sqrt{b_{j}^2+ \|\nabla F_j\|^2+\sigma^2}} +\sigma\mathbb{E}_{\xi_j} \left[ \frac{\|G_{j}\|^2}{b_{j+1}^2} \right]  .\label{eq:keyinq_a}
\end{align}
where the first term in the last inequality  is due to the fact that $$ \left| \|G_j\|- \|\nabla F_j\| \right|\leq \|G_j- \nabla F_j\| .$$
Similarly, applying the inequality  $a b \leq \frac{\lambda}{2} a^2 + \frac{1}{2\lambda} b^2$  with $\lambda=\frac{2}{\sqrt{b_{j}^2+ \|\nabla F_j\|^2+\sigma^2 }}$, $a=\frac{\sigma\|G_j\|}{b_{j+1}}$, and $b= \frac{\|\nabla F_j\| }{\sqrt{b_{j}^2+ \|\nabla F_j\|^2+\sigma^2 }}$, the second term of the right hand side in equation \eqref{eq:keyinq1} is bounded by
\begin{align}
 \mathbb{E}_{\xi_j} \left[\frac{\sigma\| \nabla F_{j}\|\| G_{j}\|}{b_{j+1}\sqrt{b_{j}^2+ \|\nabla F_j\|^2+\sigma^2 }}  \right] 
 \leq  \sigma\mathbb{E}_{\xi_{j}} \left[\frac{\|G_{j}\|^2}{b_{j+1}^2} \right]+ \frac{  \|\nabla F_{j}\|^2}{4\sqrt{b_{j}^2+ \|\nabla F_j\|^2+\sigma^2 }}. \label{eq:keyinq_b}
\end{align}
Thus, putting inequalities \eqref{eq:keyinq_a} and \eqref{eq:keyinq_b} back into \eqref{eq:keyinq1} gives
\begin{align*} 
 &\quad\mathbb{E}_{\xi_j} \left[  \left( \frac{1}{\sqrt{b_{j}^2+ \|\nabla F_j\|^2+\sigma^2}} - \frac{1}{b_{j+1} } \right) \langle{ \nabla F_j, G_j \rangle} \right]
\leq  2\sigma\mathbb{E}_{\xi_j} \left[ \frac{\|G_{j}\|^2}{b_{j+1}^2} \right] +  \frac{\|\nabla F_{j}\|^2}{2\sqrt{b_{j}^2+ \|\nabla F_j\|^2+\sigma^2}} 
\end{align*}
and, therefore, back to \eqref{eq:first},
\begin{align}
\frac{\mathbb{E}_{\xi_j} [F_{j+1}]-F_j}{\eta} 
\leq& \frac{\eta L}{2} \mathbb{E}_{\xi_j} \left[ \frac{\|G_j\|^2}{b^2_{j+1}} \right] +2\sigma \mathbb{E}_{\xi_j} \left[ \frac{\|G_j\|^2}{b^2_{j+1}} \right] 
  - \frac{\|\nabla F_{j}\|^2 }{ 2\sqrt{b_{j}^2+ \|\nabla F_j\|^2+\sigma^2}}  \nonumber
\end{align}
Rearranging,
\begin{align}
\frac{ \|\nabla F_{j}\|^2}{2\sqrt{b_{j}^2+ \|\nabla F_j\|^2+\sigma^2}} 
 \leq &\frac{F_j - \mathbb{E}_{\xi_j} [F_{j+1}] }{\eta}  + \frac{4\sigma+\eta L}{2}\mathbb{E}_{\xi_j} \left[ \frac{\|G_{j}\|^2}{b_{j+1}^2} \right]  \nonumber 
\end{align}
Applying the law of total expectation, we take the expectation of each side with respect to $\xi_{j-1}, \xi_{j-2}, \dots, \xi_1$, and arrive at the recursion
\begin{align*}
&\mathbb{E} \left[ \frac{ \|\nabla F_{j}\|^2}{2\sqrt{b_{j}^2+ \|\nabla F_j\|^2+\sigma^2}} \right] 
\leq \frac{\mathbb{E}[F_j] - \mathbb{E} [F_{j+1}] }{\eta}
+  \frac{4\sigma+\eta L}{2}\mathbb{E} \left[ \frac{\|G_{j}\|^2}{b_{j+1}^2} \right] . \nonumber 
\end{align*}
Taking $j=N$ and summing up from $k= 0$ to $k = N-1$,
\begin{align}
\sum_{k=0}^{N-1}  \mathbb{E} \left[ \frac{ \|\nabla F_{k}\|^2}{ 2\sqrt{b_{k}^2+ \|\nabla F_k\|^2+\sigma^2}} \right]   
&\leq\frac{ F_{0} - F^{*}}{\eta}+ \frac{4\sigma+\eta L}{2}\mathbb{E}\sum_{k=0}^{N-1}   \left[ \frac{\|G_{k}\|^2}{b_{k+1}^2} \right] \nonumber\\
&\leq\frac{ F_{0} - F^{*}}{\eta}+ \frac{4\sigma+\eta L}{2} \log \left( 10+\frac{ 20N\left( \sigma^2 + \gamma^2\right)}{b_0^2} \right)  
\label{eq:lips_sum}
\end{align}
where the second inequality we apply Lemma \eqref{lem:logsum} and then Jensen's inequality to bound the summation:
\begin{align}
\mathbb{E}\sum_{k=0}^{N-1}   \left[ \frac{\|G_{k}\|^2}{b_{k+1}^2} \right] &\leq  \mathbb{E} \left[1+ \log \left(1+\sum_{k=0}^{N-1}\|G_{k}\|^2/b_0^2 \right) \right] \nonumber \\
&\leq  \log \left( 10+\frac{ 20N\left( \sigma^2 + \gamma^2\right)}{b_0^2}  
\right) . \label{eq:logsum} 
\end{align}
since
\begin{align}
\mathbb{E}  \left[ b_{k}^2-b_{k-1}^2\right] &\leq  \mathbb{E}\left[\| G_k\|^2\right] \nonumber \\
&\leq 2\mathbb{E}\left[\| G_k-\nabla F_k\|^2\right] +2\mathbb{E}\left[\| \nabla F_k\|^2\right]  \nonumber \\
&\leq  2 \sigma^2 + 2\gamma^2. \label{eq:Gbound}
\end{align}

\subsubsection{Finishing the proof of the first bound in Theorem \ref{thm:improveSGDmain}}
For the term on left hand side in equation \eqref{eq:lips_sum}, we apply H\"{o}lder's inequality,
\begin{align*}
 \frac{\mathbb{E}|XY|}{\left(\mathbb{E}|Y|^{3} \right)^{\frac{1}{3}} } &\leq \left(\mathbb{E} |X|^{\frac{3}{2}} \right)^ {\frac{2}{3} }   \\
\text{with } X= \left(\frac{  \|\nabla F_{k}\|^2}{\sqrt{b_{k}^2+ \|\nabla F_k\|^2+\sigma^2}} \right)^{\frac{2}{3}}  
 \text{ and }  & Y =\left(\sqrt{b_{k}^2+ \|\nabla F_k\|^2+\sigma^2} \right)^{\frac{2}{3}} \quad \text{to obtain }
\end{align*}
  \begin{align*}  
 \mathbb{E}\left[ \frac{  \|\nabla F_{k}\|^2}{ 2\sqrt{b_{k}^2+ \|\nabla F_k\|^2+\sigma^2}}  \right]
   & \geq   \frac{ \left( \mathbb{E}  \|\nabla F_{k}\|^{
   \frac{4}{3}} \right)^\frac{3}{2} }{ 2 \sqrt{ \mathbb{E} \left[b_{k}^2+ \|\nabla F_k\|^2+\sigma^2\right] } }
   & \geq  \frac{ \left( \mathbb{E}  \|\nabla F_{k}\| ^{\frac{4}{3}} \right)^\frac{3}{2}
}{2\sqrt{b_0^2+2(k+1) (\gamma^2 +\sigma^2)} }
\end{align*}
where the last inequality is due to inequality \eqref{eq:Gbound}.
Thus \eqref{eq:lips_sum}  arrives at the inequality
\begin{align*}
&\frac{N \min_{k\in[ N-1]}  \left( \mathbb{E} \left[ \|\nabla F_{k}\| ^{\frac{4}{3}} \right]\right)^\frac{3}{2}  }{2\sqrt{b_0^2+2N (\gamma^2 +\sigma^2)} } 
 \leq  \frac{ F_{0} - F^{*} }{\eta}+  \frac{4\sigma+\eta L}{2}\left(\log \left( 1+\frac{2N\left( \sigma^2 + \gamma^2\right)}{b_0^2} \right)+1  \right).
\end{align*}
Multiplying by $ \frac{  2b_0+ 2\sqrt{2N}(\gamma +\sigma)}{  N}$, the above inequality gives
\begin{align*}
      \min_{k \in [N-1]} \left( \mathbb{E} \left[ \|\nabla F_{k}\|^{
   \frac{4}{3}} \right] \right)^\frac{3}{2}  \leq  \underbrace{\left( \frac{2b_0}{N} + \frac{4(\gamma+\sigma)}{\sqrt{N}}\right)C_{F}}_{C_N}
\end{align*}
where $$C_{F} = \frac{ F_{0} - F^{*} }{\eta}+ \frac{4\sigma+\eta L}{2}\log \left(\frac{20N\left( \sigma^2 + \gamma^2\right)}{b_0^2} +10 \right) 
. $$
Finally, the bound  is obtained by Markov's Inequality:
\begin{align*}
\mathbb{P} \left(\min_{k\in[ N-1]} \|\nabla F_{k}\|^2\geq \frac{C_N}{\delta^{3/2}}\right) 
=&\mathbb{P} \left(\min_{k\in[ N-1]} \left(\|\nabla F_{k}\|^2 \right)^{2/3} \geq \left(  \frac{C_N}{\delta^{3/2}} \right)^{2/3}\right) \\
\leq&\delta \frac{ \mathbb{E}\left[\min_{k\in[ N-1]}   \|\nabla F_{k}\|^{4/3}   \right]}{ C_N^{2/3}} \\
\leq& \delta
\end{align*}
where  in the second step  Jensen's inequality is applied to the concave function $\phi(x) = 
\min_k h_k (x)$.

\subsubsection{Finishing the proof of the second bound in Theorem \ref{thm:improveSGDmain}}
First, observe with probability $1-\delta'$ that 
$$ \sum_{i=0}^{N-1} \|\nabla F_i -G_i\|^2 \leq \frac{N \sigma^2}{\delta'}.$$
For the denominator on the left hand side of the inequality \ref{eq:lips_sum}, we let $Z=\sum_{k=0}^{N-1} \|\nabla F_{k}\|^2$ and so 
\begin{align*}
b_{N-1}^2+ 2( \|\nabla F_{N-1}\|^2 +\sigma^2) 
=&b_0^2+\sum_{i=0}^{N-2} \|G_i\|^2+ 2( \|\nabla F_{N-1}\|^2 +\sigma^2) \\
 \leq &b_0^2+2\sum_{i=0}^{N-1} \|\nabla F_i\|^2 + 2 \sum_{i=0}^{N-2} \|\nabla F_i -G_i\|^2  +2\sigma^2\\
 \leq& b_0^2+2Z+2N\frac{\sigma^2}{\delta'}
\end{align*}
Thus, we further simplify inequality \eqref{eq:lips_sum},
\begin{align*}
 &\mathbb{E} \left[ \frac{ \sum_{k=0}^{N-1} \|\nabla F_{k}\|^2}{ 2\sqrt{b_{N-1}^2+ \|\nabla F_{N-1}\|^2+\sigma^2}} \right]
 \leq \frac{ F_{0} - F^{*} }{\eta}+  \frac{4\sigma+\eta L}{2}\log \left( 10+\frac{20N\left( \sigma^2 + \gamma^2\right)}{b_0^2} \right)  \triangleq {C_F}
\end{align*}
we have with probability  $1-\hat{\delta}-\delta'$ that
\begin{align*}
\frac{C_F}{\hat{\delta}} &\geq \frac{ \sum_{k=0}^{N-1} \|\nabla F_{k}\|^2}{ 2\sqrt{b_{N-1}^2+ \|\nabla F_{N-1}\|^2+\sigma^2}}  
\geq   \frac{Z}{2 \sqrt{b_0^2+2Z+2N\sigma^2/\delta'}} 
 \label{eq:key1}
\end{align*}
That is equivalent to solve the following quadratic equation 
\begin{align*}
Z^2 - \frac{ 8C_F^2}{\hat{\delta}^2}Z -\frac{ 4C_F^2}{\hat{\delta}^2} \left( b_0^2+\frac{2N\sigma^2}{\delta'} \right) \leq 0 
\end{align*} 
which gives
\begin{align*}
Z &\leq \frac{ 4C_F^2}{\hat{\delta}^2}+  \sqrt{ \frac{ 16C_F^4}{ \hat{\delta}^4} +  \frac{ 4C_F^2}{\hat{\delta}^2} \left( b_0^2+\frac{2N\sigma^2}{\delta'} \right)}\\
&\leq \frac{ 8C_F^2}{\hat{\delta}^2}+    \frac{ 2C_F}{\hat{\delta}} \left( b_0 +\frac{\sqrt{2N}\sigma}{ \sqrt{\delta'}} \right)
 \end{align*} 
Let $\hat{\delta} = \delta'=\frac{\delta}{2}$. Replacing $Z$ with $\sum_{k=0}^{N-1} \|\nabla F_{k}\|^2$ and dividing both side with $N$ we have  with probability  $1-\delta$ 
$$\min_{k\in[N-1]}  \|\nabla F_{k}\|^2 \leq \frac{4C_F}{ N \delta } \left( \frac{8C_F}{\delta} +  2b_0 \right)+ \frac{8\sigma C_F}{ \delta^{3/2} \sqrt{N}} .$$

\end{proof}

\section{Proof of Theorem \ref{thm:converge}}
\label{sec:batch}

\subsection{Lemmas}
We will use the following lemma to argue that after an initial number of steps $N =  \lceil{ \frac{(\eta L)^2-b_0^2}{ \varepsilon} \rceil}+1$, either we have already reached a point $x_k$ such that $\| \nabla F(x_k) \|^2 \leq \varepsilon$, or else $b_N \geq \eta L$.
\begin{lem}
\label{lem:increase}
Fix $\varepsilon \in (0,1]$ and $C > 0$.   For any non-negative $a_0, a_1, \dots, $ the dynamical system
$$
b_0 > 0; \quad  \quad b_{j+1}^2 = b_j^2 +  a_j
$$
has the property that after 
{$N = \lceil{ \frac{C^2-b_0^2}{ \varepsilon} \rceil}+1$}
iterations, either $\min_{k=0:N-1}  a_k  \leq \varepsilon$, or $b_{N} \geq  \eta L$. 
\end{lem}
\begin{proof}
If $b_0 \geq  \eta C$, we are done. Else $b_0< C$.   Let $N$ be the smallest integer such that $N \geq \frac{C^2 - b_0^2}{\varepsilon}$. Suppose $ b_{N} < C$. Then
\begin{align*}
C^2 > b_{N}^2 &= b_0^2 + \sum_{k=0}^{N-1} a_k >b_{0}^2+N \min_{k \in [N-1]} a_k \quad \Rightarrow \quad  \min_{k \in [N-1]} a_k  \leq  \frac{C^2 - b_0^2}{N}
\end{align*} Hence, for $N \geq \frac{  C^2 - b_0^2}{ \varepsilon} $,
$
\min_{k\in[N-1]} a_k  \leq \varepsilon. 
$ Suppose $\min_{k\in[N-1]} a_k > \epsilon $,  then from above inequalities we have  $b_N > C$.
\end{proof}
The following Lemma  shows that  $\{F(x_k)\}_{k=0}^{\infty}$ is a bounded  sequence for any  value of $b_0>0$.  
\begin{lem}
\label{lem:stabilize1}
Suppose $F \in C_L^1$ and  $F^{*} = \inf_x F(x) > -\infty$.  
Denote by $k_0 \geq 1$ the first index such that $b_{k_0} \geq \eta L$.  Then for all $b_{k} < \eta L, k = 0,1, \ldots, k_0-1$,
\begin{align}
F_{k_0-1} - F^{*} \leq F_0- F^{*} + \frac{\eta^2L}{2} \left(1+ 2\log\left( \frac{b_{k_0-1}}{b_0}  \right)  \right)
\end{align}
\end{lem}

\begin{proof}
Suppose $k_0\geq1$  is the first index such that $b_{k_0} \geq  \eta L$.  By Lemma \ref{lem:descend}, for $ j\leq k_0-1,$
\begin{align}
F_{j+1} \leq F_{j}  -\frac{\eta}{b_{j+1}}(1 - \frac{\eta L}{2b_{j+1}}) & \| \nabla F_{j} \|^2 \leq F_{j} +  \frac{\eta^2L}{2 b_{j+1}^2} \| \nabla F_{j} \|^2 \leq F_{0} +\sum_{\ell=0}^{j}  \frac{\eta^2L}{2 b_{\ell+1}^2}  \| \nabla F_{\ell} \|^2 \nonumber \\
 \Rightarrow \quad F_{k_0-1} - F_0&\leq
 \frac{ \eta^2 L }{2}\sum_{i=0}^{{k_0}-2}  \frac{\| \nabla F_i \|^2}{b_{i+1}^2}  \nonumber\\
&\leq  \frac{ \eta^2 L }{2} \sum_{i=0}^{{k_0}-2}  \frac{(\| \nabla F_{i} \|/b_0)^2}{\sum_{\ell=0}^{i} (\| \nabla F_{\ell} \|/b_0)^2 + 1}\nonumber \\
&\leq \frac{ \eta^2 L}{2} \left(1+ \log\left(1+\sum_{\ell=0}^{k_0-2}\frac{\| \nabla F_{\ell} \|^2}{b_0^2}\right) \right)   \quad   \text{ by Lemma \ref{lem:logsum}}
\nonumber \\
&\leq \frac{ \eta^2 L}{2} \left(1+ \log \left(\frac{b^2_{k_0-1}}{b_0^2}\right)  \right) .
\nonumber 
\end{align}
\end{proof}
\subsection{Main proof}
\begin{proof}
By Lemma \ref{lem:increase}, if $\min_{k\in[N-1]} \| \nabla F(x_k) \|^2 \leq \varepsilon$ is not satisfied after 
$N = \lceil{ \frac{(\eta L)^2 - b_0^2}{ \varepsilon} \rceil}+1$ steps, then there exits a first index $1\leq k_0 \leq N$ such that $\frac{b_{k_0} }{\eta}>  L$.  
By Lemma \ref{lem:descend}, for $ j\geq 0,$

\begin{align}
F_{k_0+j} &\leq F_{k_0+j-1}- \frac{\eta}{b_{k_0+j}}(1 - \frac{\eta L}{2b_{k_0+j}}) \| \nabla F_{k_0+j-1} \|^2 \nonumber\\
&\leq F_{k_0-1}- \sum_{\ell=0}^{j} \frac{\eta}{2 b_{k_0+\ell}} \| \nabla F_{k_0+\ell-1} \|^2\nonumber\\
&\leq F_{k_0-1} -\frac{\eta}{2b_j} \sum_{\ell=0}^{j}  \| \nabla F_{k_0+\ell-1}\|^2
. \label{eq:descent1a}
\end{align}
Let  $Z= \sum_{k=k_0-1}^{M-1} \|\nabla F_{k}\|^2$, it follows that
\begin{align}
 \frac{2 \left(F_{k_0-1}- F^* \right)}{\eta} \geq  \frac{2 \left(F_0- F_{M} \right)}{\eta}  \geq  \frac{\sum_{k=k_0-1}^{M-1}\| \nabla F_{k} \|^2}{b_{M}} 
\geq  \frac{Z}{\sqrt {Z+b_{k_0-1}^2}} \nonumber .
\end{align} Solving the quadratic inequality for $Z$, 
\begin{align}
\sum_{k=k_0-1}^{M-1} \|\nabla F_{k}\| ^2\leq \frac{4 \left(F_{k_0-1}- F^* \right)^2}{\eta^2} +\frac{2 \left(F_{k_0-1}- F^* \right)b_{k_0-1}}{\eta} . \label{eq:boundF}
\end{align}

If $k_0=1$, the stated result holds by multiplying both side by $\frac{1}{M}$. Otherwise, $k_0 > 1.$ 
From Lemma \ref{lem:stabilize1}, we have 
$$
F_{k_0-1} - F^{*} \leq F_0- F^{*} +  \frac{ \eta^2L}{2} \left(1+ 2\log \left(\frac{ \eta L}{b_0}  \right)\right) .
$$
Replacing  $F_{k_0-1} - F^{*} $ in  \eqref{eq:boundF} by above bound, we have 
\begin{align*}
&\sum_{k=k_0-1}^{M-1}  \| \nabla F_{k}\|^2\\
 \leq&  \underbrace{ \left( \frac{2 \left(F_{0}- F^* \right)}{\eta}+ \eta L\left( 1+ 2\log \left({ \eta L}/{b_0}  \right) \right)\right)^2+2L \left(F_{0}- F^* \right)+  (\eta L)^2\left( 1+ 2\log \left(\frac{ \eta L}{b_0}  \right) \right) }_{C_M}
\end{align*}
Thus, we are assured that
$$
\min_{k=0:N+M-1}  \| \nabla F_{k} \|^2 \leq \varepsilon
$$
where $N \leq \frac{L^2 - b_0^2}{\varepsilon} $ and $M = \frac{C_M}{\varepsilon}$ .
\end{proof}

\section{Numerical experiments}
\label{sec:num}
 
With guaranteed convergence of AdaGrad-Norm  and its robustness to the parameters $\eta$ and $b_0$, we perform experiments on several data sets ranging from simple linear regression over Gaussian data to neural network architectures on state-of-the-art (SOTA) image data sets including ImageNet. {These experiments are not about outperforming the strong baseline of well-tuned SGD, but to further strengthen the theoretical finding that the convergence of AdaGrad-norm requires less hyper-parameter tuning while maintaining a comparable performance as the well-tuned SGD methods.}
\subsection{Synthetic data}
In this section, we consider linear regression to corroborate our analysis, i.e.,
$$
F(x) = \frac{1}{2m} \| Ax - y \|^2 = \frac{1}{(m/n)} \sum_{k=1}^{m/n}\frac{1}{2n} \| A_{\xi_k}x  - y_{\xi_k} \|^2
$$
where $A\in\mathbb{R}^{m\times d} $, $m$ is the total number of samples, $n$ is the mini-batch (small sample) size for each iteration, and $ A_{\xi_k}\in \mathbb{R}^{n\times d}$.
 Then  AdaGrad-Norm update is
\begin{align*}
x_{j+1} = x_{j} -  \frac{\eta A_{\xi_j}^{T}\left(A_{\xi_j}x_j  - y_{\xi_j} \right)/n}{\sqrt{b_0^2+\sum_{\ell=0}^j \left(\|A_{\xi_\ell}^{T}\left(A_{\xi_\ell}x_\ell  - y_{\xi_\ell} \right)\|/n \right)^2}}.
\end{align*}
\begin{figure}[h]
  \centering
   \includegraphics[width=.8\textwidth]{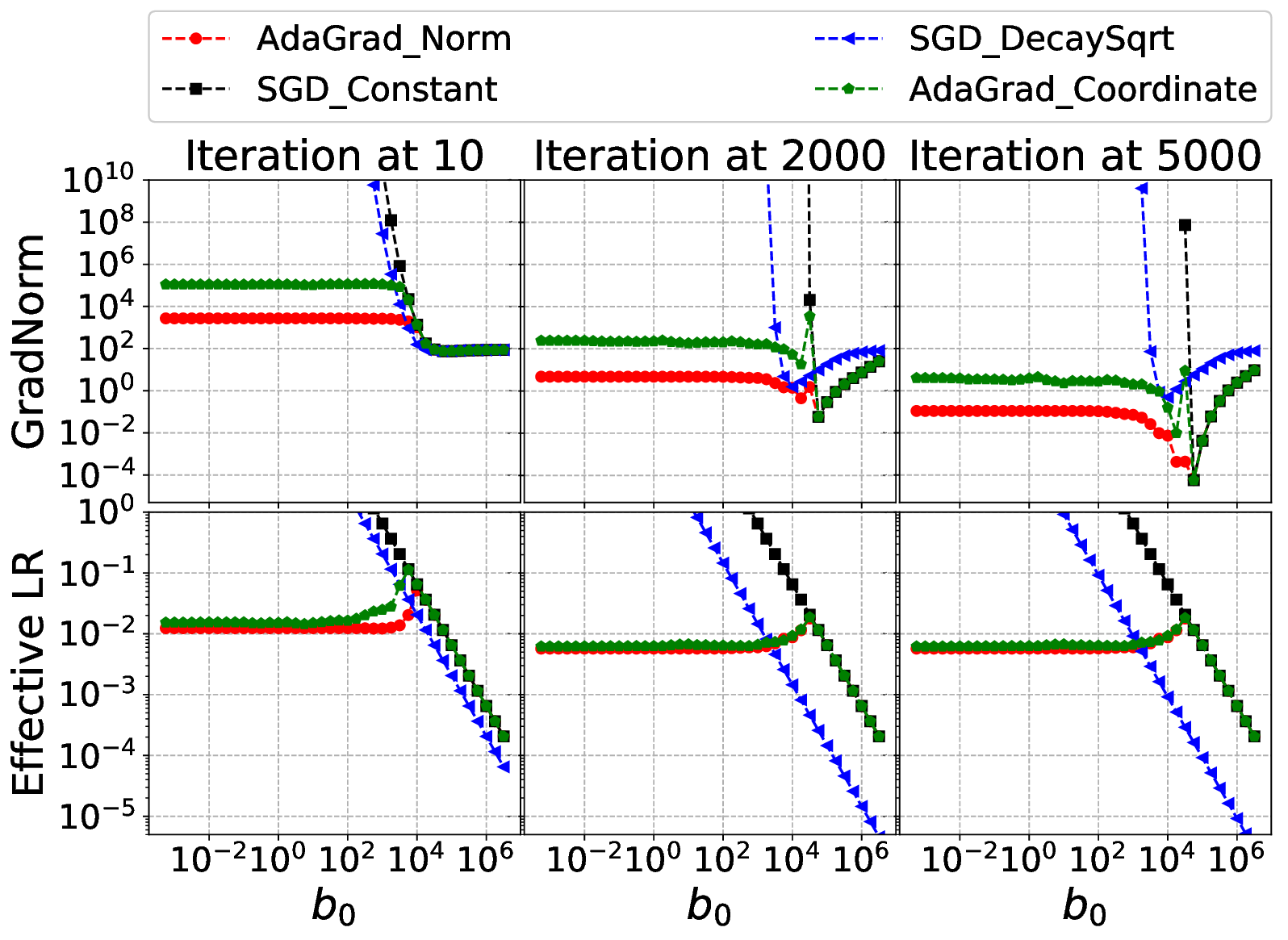}
        \caption{Gaussian Data -- Stochastic Setting. The top 3 figures plot the square of the gradient norm for  linear regression, $\|A^{T}\left(Ax_j  - y\right) \| /m$, w.r.t.  $b_0$,  at iterations 10, 2000 and 5000 (see title) respectively. The bottom 3  figures plot the corresponding effective learning rates (median of $\{b_j(\ell)\}_{\ell=1}^d$ for AdaGrad-Coordinate), w.r.t.  $b_0$,  at iteration  10, 2000 and 5000 respectively (see title). }\label{fig:linear}
  \end{figure}
We simulate $A\in \mathbb{R}^{1000\times2000}$ and $x^{*}\in \mathbb{R}^{1000}$ such that each entry of $A$ and $x^{*}$ is an i.i.d. standard Gaussian. Let $y=Ax^*$. For each iteration, we independently draw a small sample of size $n=20$ and $x_0$ whose entries follow i.i.d. uniform in $[0,1]$. The  vector $x_0$ is same for all the methods so as to eliminate the effect of random initialization in weight vector. Since $F^*=0$, we set  $\eta = F(x_0)-F^* = \frac{1}{2m}\|Ax_0-b\|^2=650$. We vary the initialization $b_0 > 0$ as to compare with plain SGD using (a) SGD-Constant: fixed stepsize $\frac{650}{b_0}$, (b) SGD-DecaySqrt: decaying stepsize  $\eta_j = \frac{650}{b_0\sqrt{j}}$, and (c) AdaGrad-Coodinate:  update the $d$ parameters $b_j(k), k=1,2,\ldots,d$ at each iteration $j$, one for each coordinate of $x_j \in \mathbb{R}^d$. Figure \ref{fig:linear} plots $\|A^{T}\left(Ax_j  - y\right) \|/m$ (GradNorm) and  the effective learning rates  at iterations $10$, $2000$, and $5000$, and as a function of $b_0$, for each of the four methods. The effective learning rates are $\frac{650}{b_j}$ (AdaGrad-Norm), $\frac{650}{b_0}$(SGD-Constant), $\frac{650}{b_0\sqrt{j}}$(SGD-DecaySqrt), and the median of $\{b_j(\ell)\}_{\ell=1}^d$ (AdaGrad-Coordinate). 
   \begin{figure}[h]
  \centering
   \includegraphics[width=.8\textwidth]{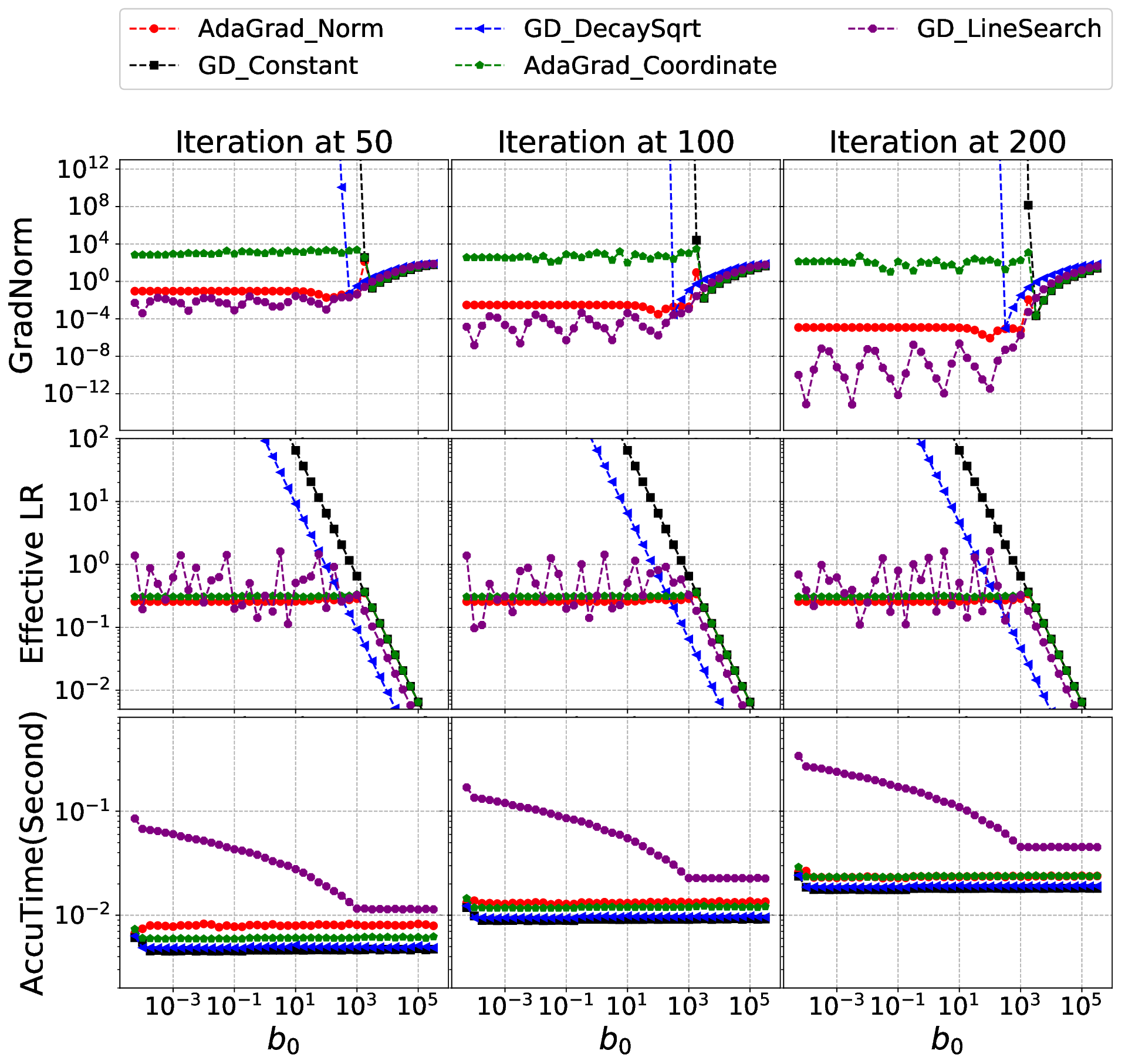}
        \caption{Gaussian Data - Batch Setting.  The y-axis and x-axis in the top and middle 3 figures are the same as in Figure 1. The bottom 3 figures plot the accumulated computational time (AccuTime) up to iteration $50$, $100$ and $200$ (see title), as a function of $b_0$. }\label{fig:batch}
  \end{figure} 
  
We can see in Figure \ref{fig:linear} how AdaGrad-Norm and AdaGrad-Coordinate auto-tune the learning rate adaptively to a certain level to match the unknown Lipschitz smoothness constant and the stochastic noise so that the gradient norm converges for a significantly wider range of $b_0$ than for either SGD method.  In particular, when $b_0$ is initialized too small, AdaGrad-Norm and AdaGrad-Coordinate still converge with good speed while SGD-Constant and SGD-DecaySqrt diverge. When $b_0$ is initialized too large (stepsize too small), surprisingly AdaGrad-Norm and AdaGrad-Coordinate converge at the same speed as  SGD-Constant. This possibly can be explained by Theorem \ref{thm:converge} because this is somewhat like the deterministic setting (the stepsize controls the variance $\sigma$ and a smaller learning rate implies smaller variance).  Comparing AdaGrad-Coordinate and AdaGrad-Norm,  AdaGrad-Norm is more robust to the initialization $b_0$ but is not better than  AdaGrad-Coordinate when the initialization $b_0$ is close to the optimal value of $L$.

Figure \ref{fig:batch} explores the batch gradient descent setting, when there is no variance $\sigma =0$ (i.e., using the whole data sample for one iteration). The experimental setup in Figure 2 is the same as Figure \ref{fig:linear} except for the sample size $m$ of each iteration. Since the line-search method (GD-LineSearch) is one of the most important algorithms in deterministic gradient descent for adaptively choosing the step-size at each iteration, we also compare to this method -- see Algorithm 2 in the appendix for our particular implementation of Line-Search.  We see that the behavior of the four methods, AdaGrad-Norm, AdaGrad-Coordinate, GD-Constant, and GD-DecaySqrt, are very similar to the stochastic setting, albeit AdaGrad-Coordinate here is worse than in the stochastic setting.  Among the five methods in the plot, GD-LineSearch performs the best but with significantly longer computational time, which is not practical in large-scale machine learning problems.

\subsection{Image data}
 \begin{figure}[h]
\centering
\includegraphics[width=0.8\columnwidth]{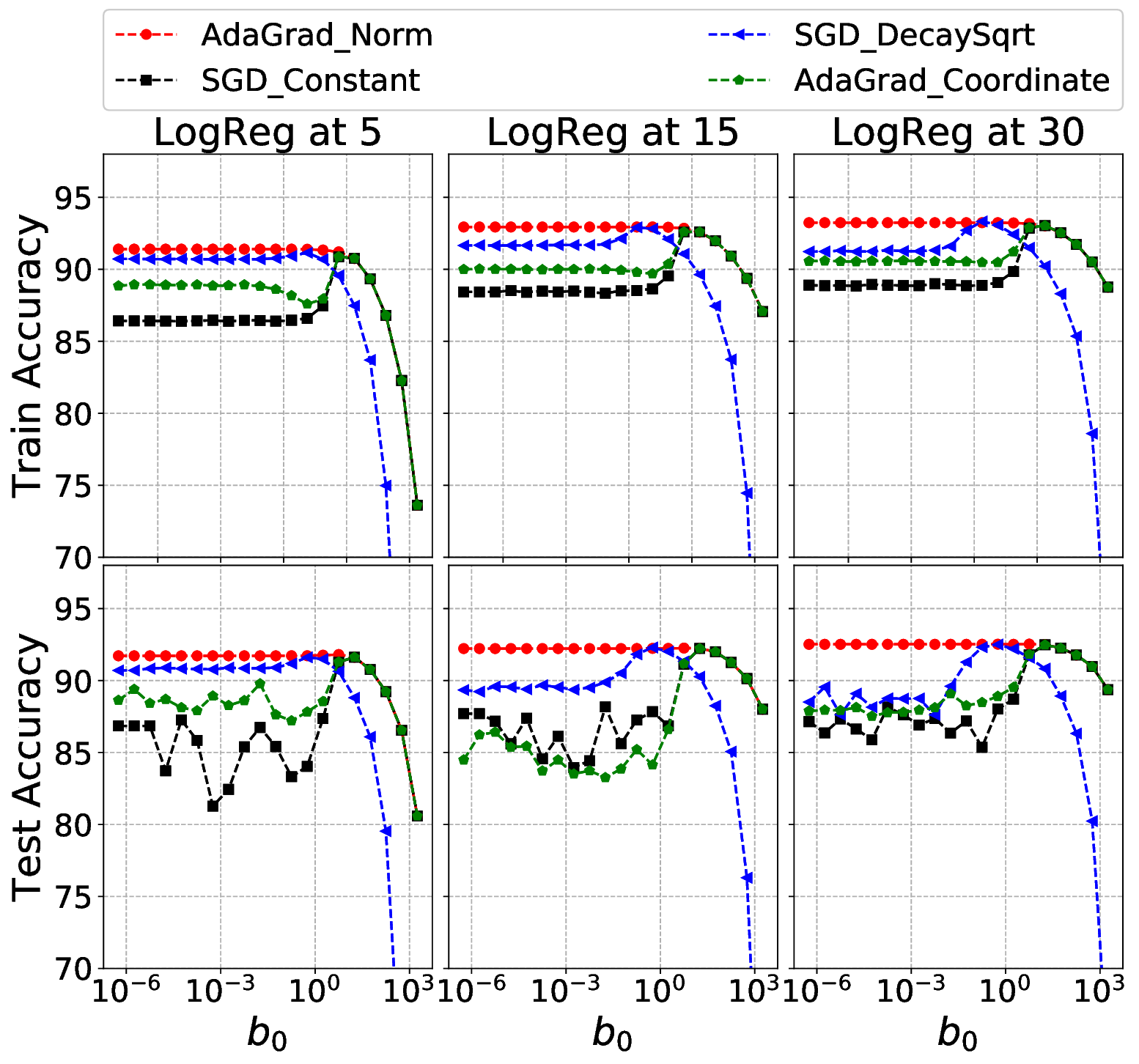}
\caption{MNIST.  In each plot, the y-axis is the train or test accuracy and the x-axis is $b_0$. The 6 plots are for logistic regression (LogReg) with average at epoch 1-5, 11-15 and 26-30. The title is the last epoch of the average. Note green and red curves overlap when $b_0$ belongs to $[10, \infty)$}
\end{figure}
 In this section, we extend our numerical analysis to the setting of deep learning and show that the robustness of AdaGrad-Norm does not come at the price of worse generalization -- an important observation that is not explained by our current theory. The experiments are done in PyTorch \citep{paszke2017automatic} and parameters are by default if no specification is provided.\footnote{The code we used is originally from \url{https://github.com/pytorch/examples/tree/master/imagenet}} We did not find it practical to compute the norm of the gradient for the entire neural network during back-propagation. Instead, we adapt a stepsize for each neuron or each convolutional channel by updating $b_j$ with the gradient of the neuron or channel. Hence, our experiments depart slightly from a strict AdaGrad-Norm method and include a limited adaptive metric component. Details in implementing AdaGrad-Norm  in a neural network are explained in the appendix and the code is also provided.\footnote{\url{https://github.com/xwuShirley/pytorch/blob/master/torch/optim/adagradnorm.py}}

\paragraph{Datasets and Models} We test on three data sets: MNIST \citep{lecun1998gradient}, CIFAR-10 \citep{krizhevsky2009learning} and ImageNet \citep{imagenet_cvpr09}, see Table 1 in the appendix for detailed descriptions.  For MNIST, our models are a logistic regression (LogReg), a multilayer network with two fully connected layers (FulConn2) with 100 hidden units and ReLU activations, and a convolutional neural network (see Table \ref{2cnn} in the appendix for details).  For CIFAR10, our model is ResNet-18  \citep{he2016deep}. For both data sets, we use 256 images per iteration (2 GPUs with 128 images/GPU, 234 iterations per epoch for MNIST and 196 iterations per epoch for CIFAR10). For ImagetNet, we use ResNet-50 and 256 images for one iteration (8 GPUs with 32 images/GPU, 5004 iterations per epoch). Note that we do not use accelerated methods such as adding momentum in the training. 

 We pick these models for the following reasons:
(1) LR with MNIST represents the smooth  loss function;
(2)  FC with MNIST represents the non-smooth loss function;
 (3) CNN with MNIST belongs to a class of simple shallow network architectures;
 (4) ResNet-18 in CIFAR10 represents a complicated network architecture involving many other added features achieving SOTA performance;
(5) ResNet-50 in ImageNet represents large-scale data and a deep network architecture.

\paragraph{Experimental Details}  In order to make the setting match our assumptions, {we make several changes, which are not practically meaningful scenarios but serve only for corroborating our theorems.}

For the experiment in MNIST,  we do not use bias, regularization (zero weight decay), dropout, momentum, batch normalization \citep{ioffe2015batch}, or any other added features that help achieving SOTA performance (see Figure 3 and Figure 4).   However, the architecture of ResNet by default is built with the 
celebrated batch normalization  (Batch-Norm) method as important layers. Batch-Norm accomplishes the auto-tuning property by normalizing the means and variances of mini-batches in a particular way during the forward-propagation, and in return is back-propagated with projection steps. This projection phenomenon is highlighted in weight normalization \citep{salimans2016weight,wwb18}. Thus, in the ResNet-18 experiment on CIFAR10, we are particularly interested in how Batch-Norm interacts with the auto-tuning property of AdaGrad-Norm. We disable the learnable scale and shift parameters in the Batch-Norm layers \footnote{Set {\it{nn.BatchNorm2d(planes,affine=False)}} } and  compare the default setup in ResNet \citep{goyal2017accurate}. The resulted plots are located in Figure 4 (bottom left and bottom right). In the ResNet-50 experiment on ImageNet, we also depart from the standard set-up by initializing the weights of the last fully connected layer with i.i.d. Gaussian samples with mean zero and variance $0.03125$. {Note that the default initialization 
for the last fully-connected layer of ResNet50 is an i.i.d. Gaussian distribution with mean zero and variance of $0.01$. Instead, we use variance $0.03125$ in that the AdaGrad-Norm algorithm takes the norm of the gradient. The initialization of Gaussian distribution with higher variance results in larger accumulative gradient norms, which is likely to make AdaGrad-Norm robust to small initialization of $b_0$. To some extent, AdaGrad-Norm could be sensitive to the model's initialization. But how much sensitive the AdaGrad-Norm, or more generally the adaptive gradient methods, to the initialization of the model could be a potential future direction.}

 \begin{figure}[h]
\centering
\includegraphics[width=0.495\columnwidth]{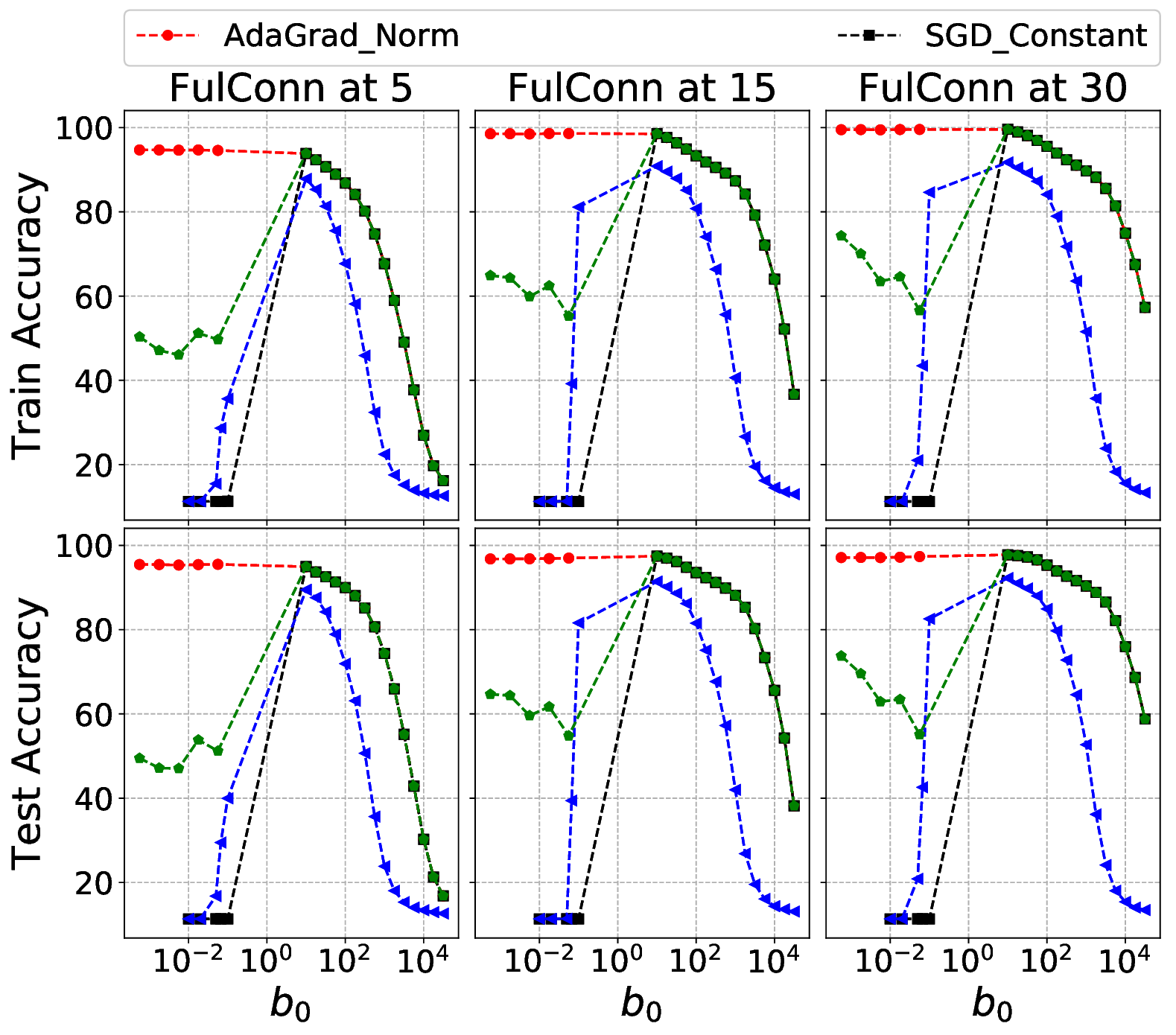} 
\hfill
\includegraphics[width=.495\columnwidth]{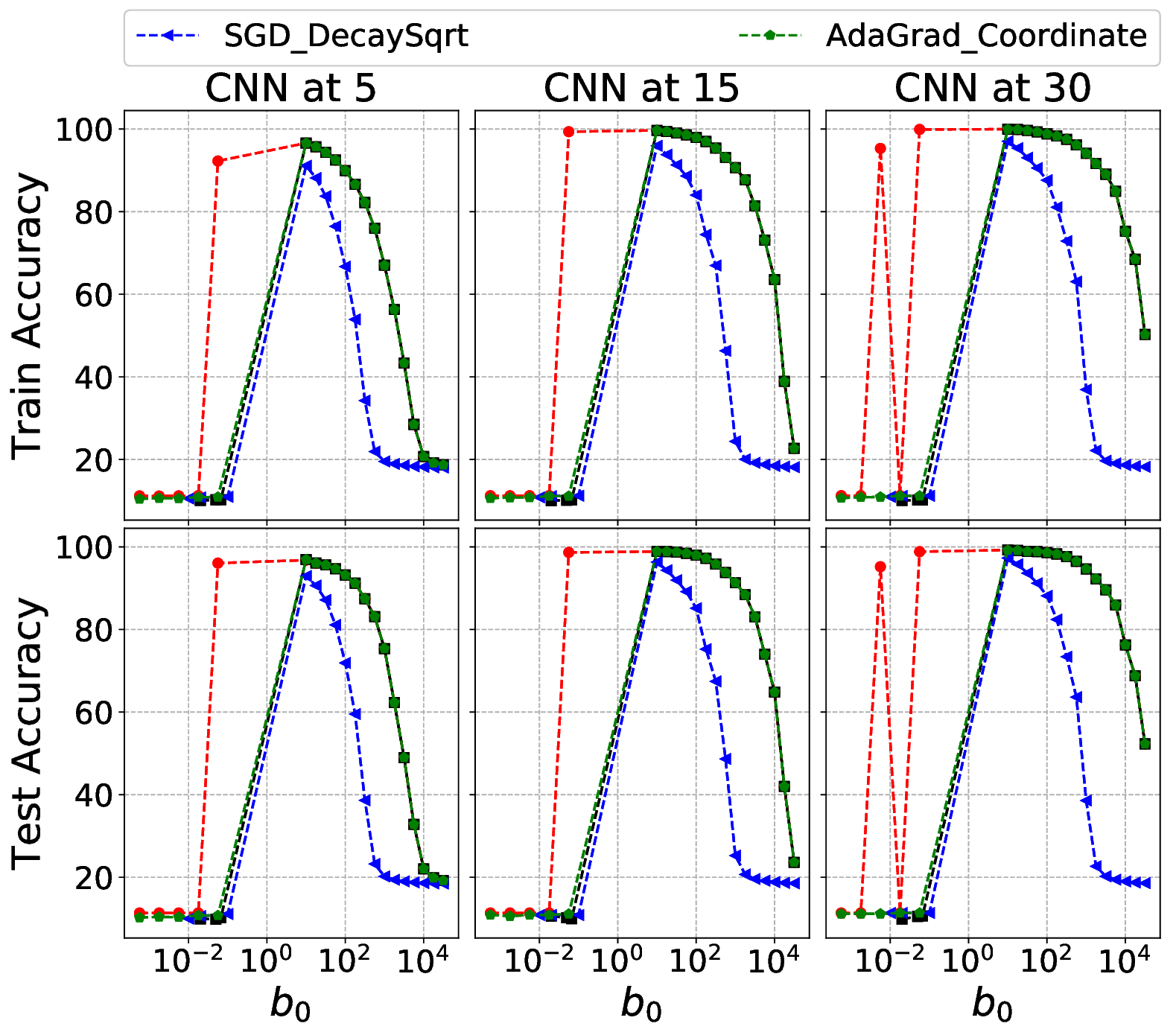}
 \centering
\includegraphics[width=.495\columnwidth]{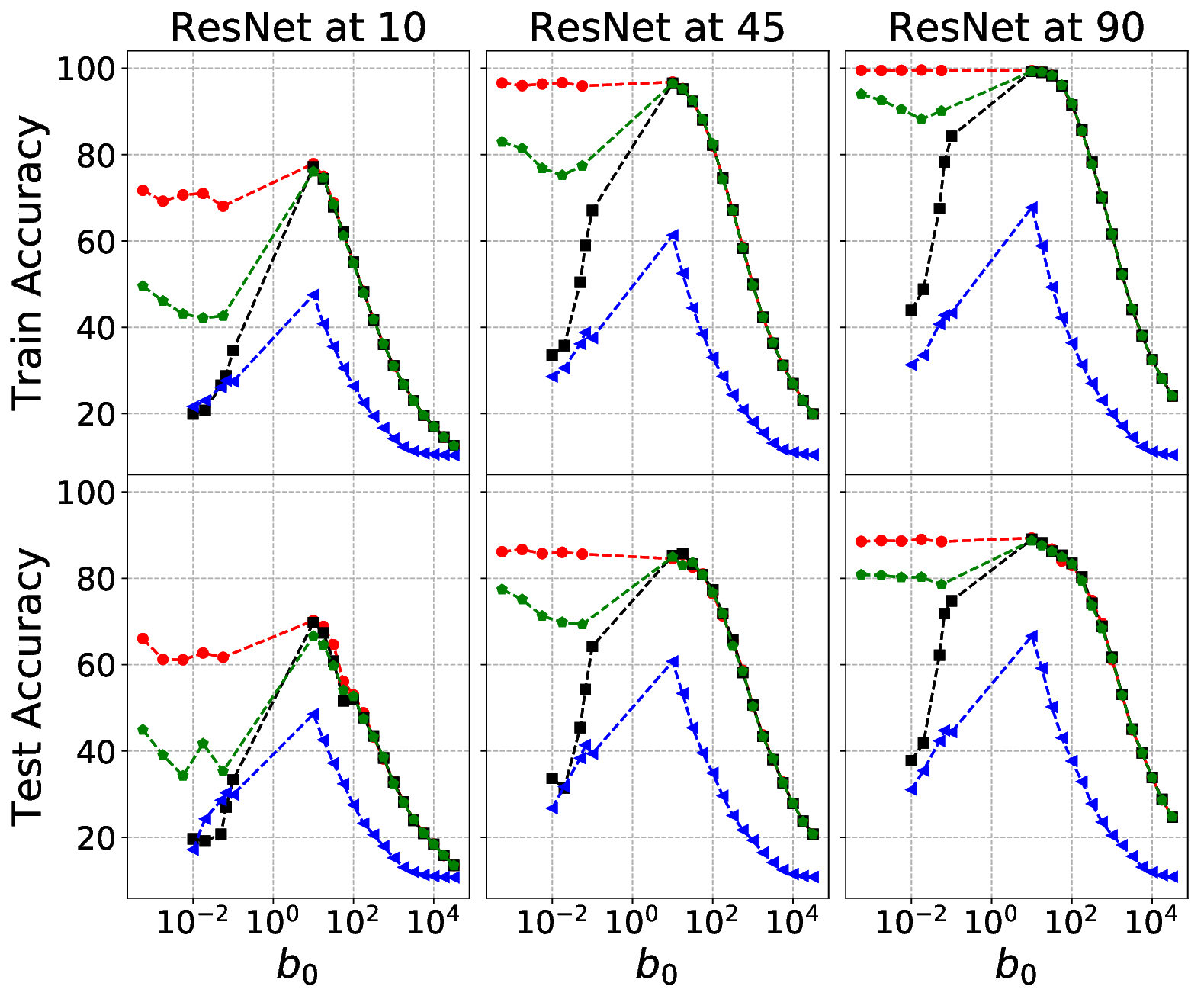}
\hfill
\includegraphics[width=0.495\columnwidth]{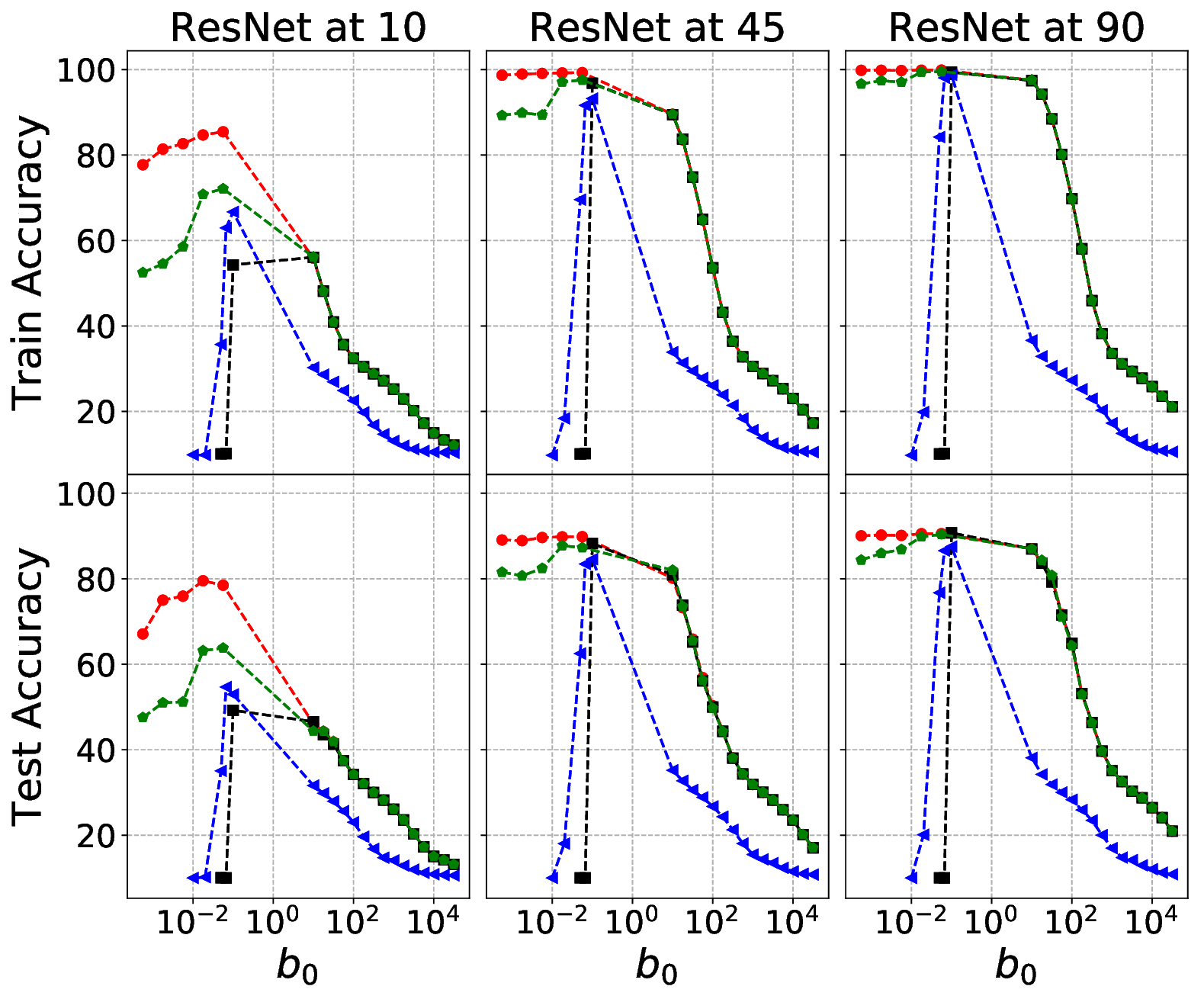}
\caption{In each plot, the y-axis is the train or test accuracy and the x-axis is $b_0$.  Top left 6 plots are for MNIST using the two-layer fully connected network (ReLU activation). Top right  6 plots are for MNIST using convolution neural network (CNN). Bottom left  6 plots are for CIFAR10 using ResNet-18 with disabling learnable parameter in Batch-Norm.  Bottom right  6 plots are for CIFAR10 using ResNet-18 with default Batch-Norm. The points in the (top) bottom plot are the average of epoch (1-5) 6-10, epoch (11-15) 41-45 or epoch (26-30)  86-90. The title is the last epoch of the average. Note green, red  and black curves overlap when $b_0$ belongs to $[10, \infty)$. Better read on screen. }
\end{figure}

For all experiments, same initialized vector  $x_0$ is used for the same model so as to eliminate the effect of random initialization in weight vectors.  We set $\eta=1$ in all AdaGrad implementations, noting that in all these problems we know that $F^{*} = 0$ and we measure that $F(x_0)$ is between 1 and 10.   Indeed, we approximate the loss using a sample of 256 images to be $\frac{1}{256}\sum_{i=1}^{256}f_i(x_0)$: $2.4129$ for logistic regression, $2.305$ for two-layer fully connected model, $2.301$ for  convolution neural network, $2.3848$ for ResNet-18 with disable learnable parameter in Batch-Norm, $2.3459$ for ResNet-18 with default Batch-Norm, and $7.704$ for ResNet-50. We vary the initialization $b_0$ while fixing all other parameters and  plot the training accuracy and testing accuracy after different numbers of epochs.   We compare AdaGrad-Norm with initial parameter $b_0$ to (a) SGD-Constant: fixed stepsize $\frac{1}{b_0}$,  (b) SGD-DecaySqrt:  decaying stepsize  $\eta_j = \frac{1}{b_0\sqrt{j}}$ (c) AdaGrad-Coordinate: a vector of per-coefficient stepsizes. \footnote{We use \it{torch.optim.adagrad}}

\begin{figure}[ht]
 \centering
\includegraphics[width=.8\columnwidth]{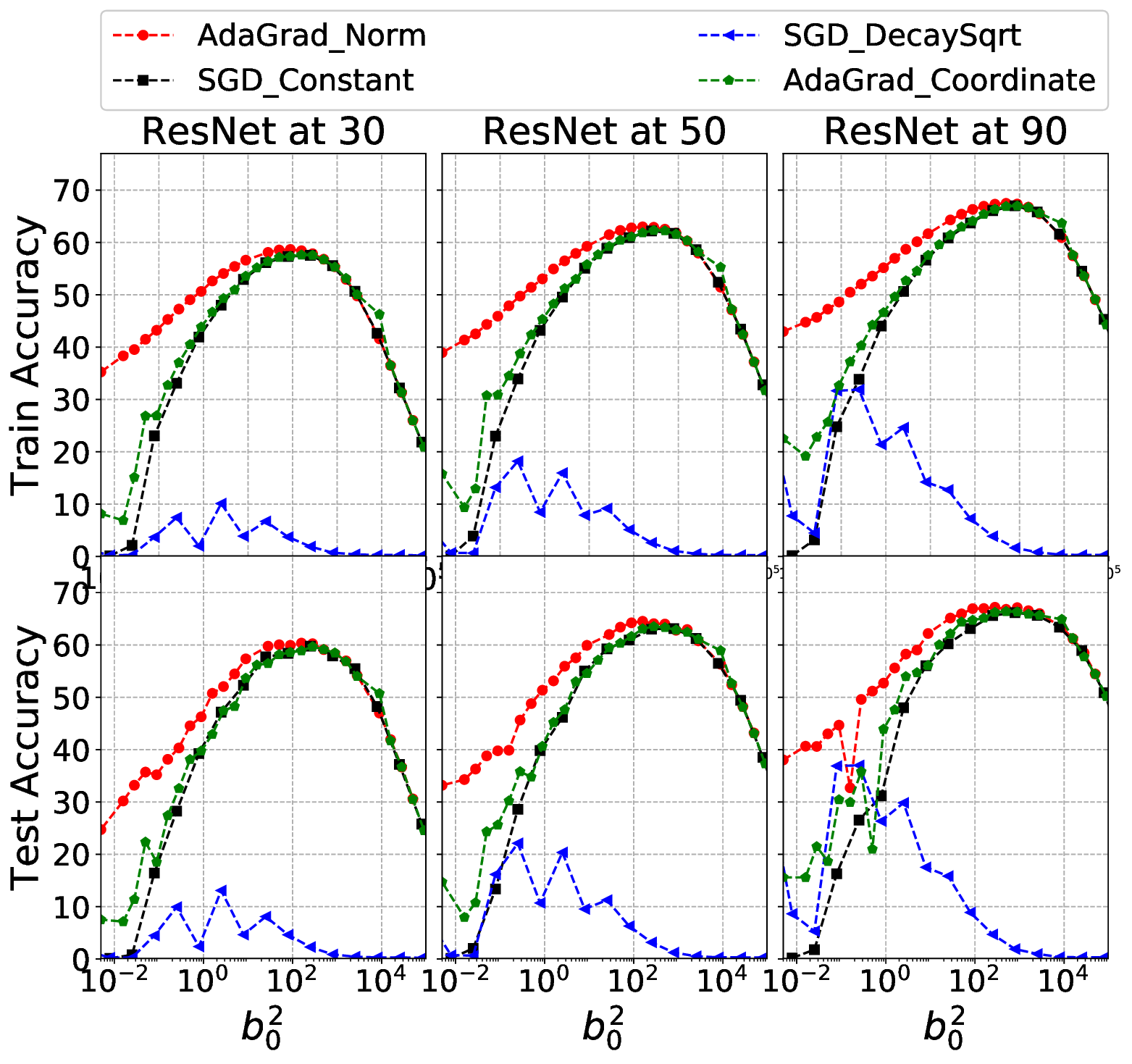}
\caption{ImageNet trained with model ResNet-50. The y-axis is the average train or test accuracy at epoch 26-30, 46-50, 86-90 w.r.t. $b_0^2$. Note no momentum is used in the training. See \textbf{Experimental Details}. Note green, red  and black curves overlap when $b_0$ belongs to $[10, \infty)$.}
\end{figure}
 \begin{figure}[ht]
\centering
\includegraphics[width=.495\columnwidth]{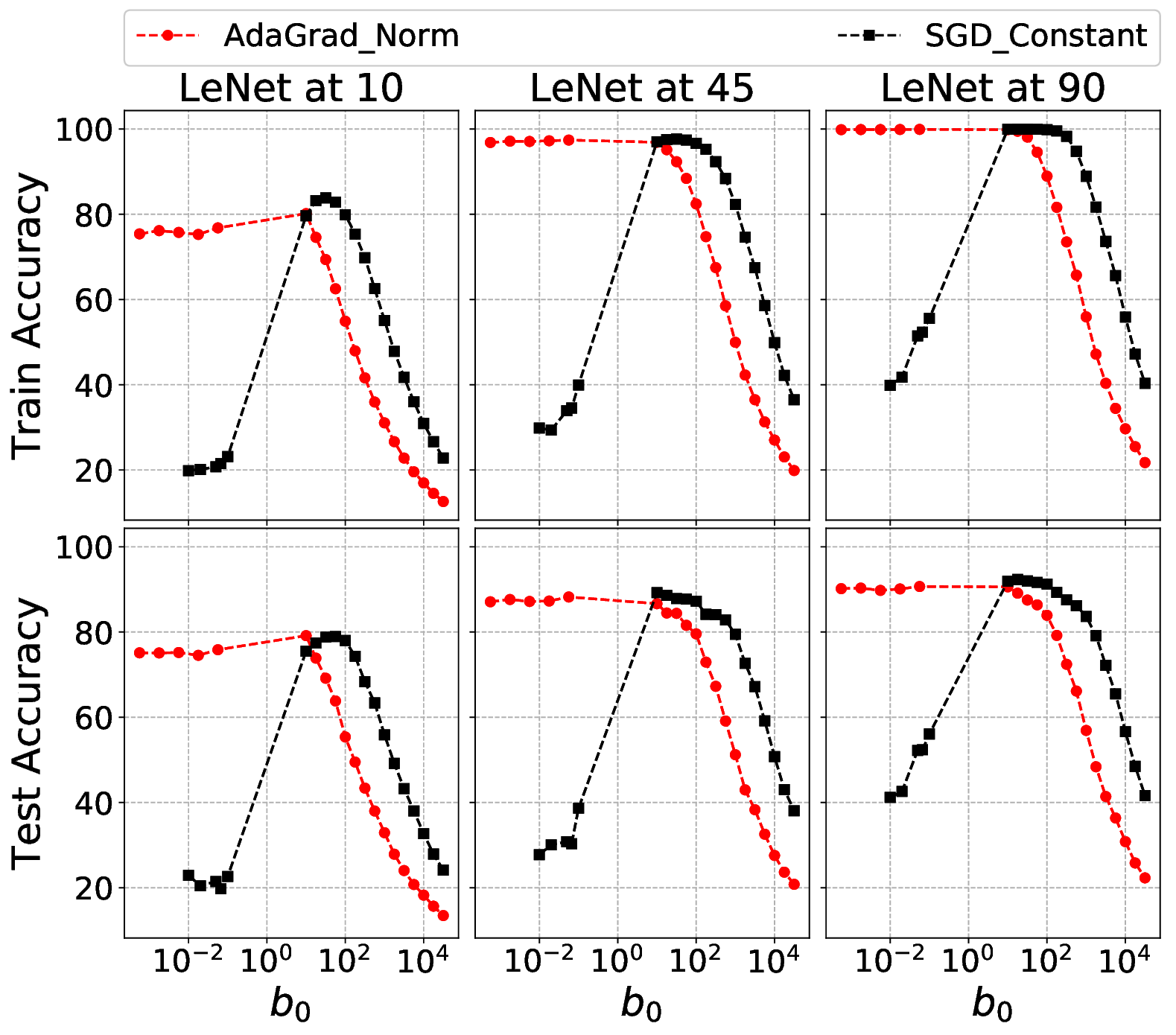}
\hfill
\includegraphics[width=0.495\columnwidth]{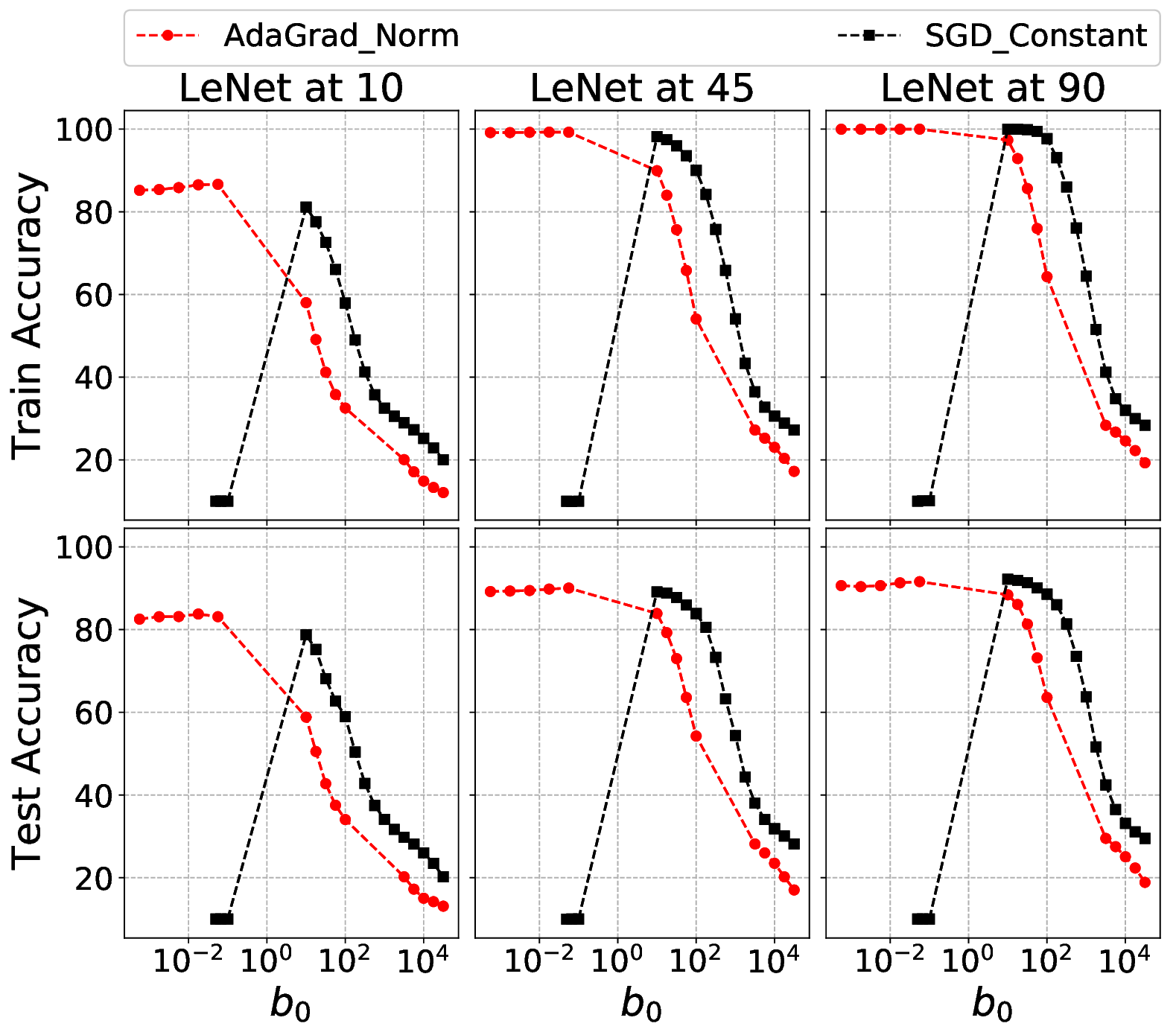}
\caption{The performance of  SGD and AdaGrad-Norm in presence of momentum (see Algorithm 3). In each plot, the y-axis is train or test accuracy and x-axis is $b_0$.  Left  6 plots are for CIFAR10 using ResNet-18 with disabling learnable parameter in Batch-Norm.  Right  6 plots are for CIFAR10 using ResNet-18 with default Batch-Norm.  The points in the plot are the average of epoch 6-10, epoch 41-45 and epoch 86-90, respectively. The title is the last epoch of the average. Better read on screen.}\label{fig:momentum}
\end{figure}

\paragraph{Observations and Discussion}  {In all experiments shown in  Figures 3, 4, and 5, we fix $b_0$ and compare the accuracy for the four algorithms; the convergence of AdaGrad-Norm is much better even for small initial values $b_0$, and shows much stronger robustness than the alternatives. In particular, Figures 3 and 4 show that the AdaGrad-Norm's  accuracy is extremely robust (as good as the best performance) to the choice of $b_0$. At the same time, the SGD methods and AdaGrad-Coordinate are highly sensitive. For Figure 5, the range of parameters $b_0$ for which AdaGrad-Norm attains
its best performance is also larger than the corresponding range for
SGD-Constant and AdaGrad-Coordinate but sub-optimal for small values of $b_0$. It is likely to indicate that for ImageNet training, AdaGrad-Norm does not remove the need to tune $b_0$ but makes the hyper-parameter search for $b_0$ easier. Note that the best test accuracy in Figure 5 is substantially lower than numbers in the
literature, where optimizers for ResNet-50 on ImageNet attain test accuracy around $76\%$ \citep{goyal2017accurate}, about $10\%$ better than the best result in Figure 5. This is mainly because (a) we do not apply momentum methods, and perhaps more critically (b) both SGD and AdaGrad-Norm do not use the default decaying  scheduler for $\eta$ as in  \cite{goyal2017accurate}. Instead, we use a constant rate $\eta=1$. Our purpose is not to achieve the comparable state-of-the-art results but mainly to verify that AdaGrad-Norm is less sensitive to hyper-parameter and requires less hyper-parameter tuning.}

Similar to the Synthetic Data, when $b_0$ is initialized in the range of well-tuned stepsizes, AdaGrad-Norm gives almost the same accuracy as SGD with constant stepsize; when $b_0$ is initialized too small, AdaGrad-Norm still converges with good speed (except for CNN in MNIST), while SGDs do not. The divergence of AdaGrad-Norm with small $b_0$ for CNN in MNIST (Figure 4, top right) can be possibly explained by the unboundedness of gradient norm in the four-layer CNN model. In contrast, the 18-layer or 50-layer ResNet model is very robust to all range of $b_0$ in experiments (Figure 4, bottom), which is due to Batch-Norm that we further discuss in the next paragraph.
 
We are interested in the experiments of Batch-Norm by default and Batch-Norm without learnable parameters because we want to understand how AdaGrad-Norm interacts with models that already have the built-in feature of auto-tuning stepsize such as Batch-Norm. 
First, comparing the outcomes of Batch-Norm with the default setting (Figure 4, bottom right) and without learnable parameters (Figure 4, bottom left), we see the learnable parameters  (scales and shifts)  in Batch-Norm can be very helpful in accelerating the training.  Surprisingly, the best stepsize in Batch-Norm with default for SGD-Constant is at $b_0=0.1$ (i.e., $\eta =10$). While the learnable parameters are more beneficial to AdaGrad-Coordinate, AdaGrad-Norm seems to be affected less.  Overall, combining the two auto-tuning methods (AdaGrad-Norm and  Batch-Norm) give good performance. 

At last, we add momentum to the stochastic gradient descent methods as empirical evidence to showcase the robustness of adaptive methods with momentum shown in Figure \ref{fig:momentum}. Since SGD with  $0.9$ momentum is commonly used, we also set $0.9$ momentum for our implementation of AdaGrad-Norm. See Algorithm 3 in the appendix for details. The results (Figure 6) show that AdaGrad-Norm with momentum is highly robust to initialization while SGD with momentum is not. SGD with momentum does better than AdaGrad-Norm when the initialization  $b_0$ is greater than the Lipschitz smoothness constant. When $b_0$ is smaller than the Lipschitz smoothness constant, AdaGrad-Norm performs as well as SGD with the best stepsize ($0.1$).

\section*{Acknowledgments}
 Special thanks to Kfir Levy for pointing us to his work, to Francesco Orabona for reading a previous version and pointing out a mistake, and to Krishna Pillutla for discussion on the unit mismatch in AdaGrad.  We thank Arthur Szlam and  Mark Tygert for constructive suggestions. We also thank Francis Bach, Alexandre Defossez, Ben Recht, Stephen Wright, and Adam Oberman. We appreciate the help with the experiments from Priya Goyal, Soumith Chintala, Sam Gross, Shubho Sengupta, Teng Li, Ailing Zhang, Zeming Lin, and Timothee Lacroix. Finally, we owe particular gratitude to the reviewers and the editor for their suggestions and comments that significantly improved the paper.

\bibliography{example_paper}

\appendix
\section{Tables}

\begin{table}[H]
\caption{Statistics of data sets. DIM is the dimension of a sample}
\label{data-table}
\begin{center}
\begin{small}
\begin{sc}
\begin{tabular}{lcccr}
\toprule
Dataset &  Train  & Test & Classes & Dim\\
\midrule
MNIST    &   60,000& 10,000& 10& 28$\times$28 \\
CIFAR-10    & 50,000 & 10,000 &10 & 32$\times$32\\
 ImageNet    & 1,281,167 & 50,000 & 1000 & Various\\
\bottomrule
\end{tabular}
\end{sc}
\end{small}
\end{center}
\end{table}
\begin{table}[H]
\caption{Architecture for four-layer convolution neural network (CNN)}
\label{2cnn}
\begin{center}
\begin{small}
\begin{sc}
\begin{tabular}{cccr}
\toprule
Layer type &  Channels & Out Dimension\\
$5\times 5$ conv relu & 20 & 24\\
$2\times 2$ max pool, str.2 & 20  & 12\\
$5\times 5$ conv relu & 50 & 8\\
$2\times 2$ max pool, str.2 & 50 & 4\\
FC  relu& N/A & 500\\
FC  relu& N/A & 10\\
\bottomrule
\end{tabular}
\end{sc}
\end{small}
\end{center}
\end{table}

\section{Implementing Algorithm 1 in a neural network}
In this section, we give the details for implementing our algorithm in a neural network. In the standard neural network architecture, the computation of each neuron consists of an elementwise nonlinearity  of a linear transform of input features or output of previous layer:
\begin{equation}
y = \phi (\langle{w,x\rangle}+b)  \label{eqa:neuron},
\end{equation}
where $w$ is the $d$-dimensional weight vector, $b$ is a scalar bias term, $x$,$y$ are respectively a $d$-dimensional vector of input features (or output of previous layer) and the output of current neuron, $\phi(\cdot)$ denotes an element-wise nonlinearity.  

For fully connected layer,   the stochastic gradient $G$  in Algorithm 1 represents the gradient of the current neuron (see the green curve,  Figure \ref{fig:M1}).  Thus, when implementing our algorithm in PyTorch, AdaGrad Norm is one learning rate associated to one neuron for fully connected layer, while SGD has one learning rate for all neurons. 

For convolution layer, the stochastic gradient $G$  in Algorithms 1 represents the gradient of each channel  in the neuron. For instance, there are 6 learning rates for the first layer in the LeNet architecture (Table 1).  Thus,  AdaGrad-Norm is one learning rate associated to one channel. 

\def\layersep{1.5cm}
\begin{figure}[H]
\centering
\begin{tikzpicture}[
   shorten >=1pt,->,
   draw = black!50,
    node distance=\layersep,
    every pin edge/.style={<-,shorten <=1pt},
    neuron/.style={circle,fill=black!25,minimum size=17pt,inner sep=0pt},
    input neuron/.style={neuron, fill=black!50},
    output neuron/.style={neuron, fill=black!50},
    hidden neuron/.style={neuron, fill=black!50},
    annot/.style={text width=4em, text centered}
]

    \foreach \name / \y in {1,...,4}
        \node[input neuron, pin=left:Dim \y] (I-\name) at (0,-\y) {};

    \newcommand\Nhidden{2}

    \foreach \N in {1,...,\Nhidden} {
       \foreach \y in {1,...,5} {
          \path[yshift=0.5cm]
              node[hidden neuron] (H\N-\y) at (\N*\layersep,-\y cm) {};
           }
   \node[annot,above of=H\N-1, node distance=1cm] (hl\N) {Hidden layer \N};
    }

    \node[output neuron,pin={[pin edge={<-}]right: loss}, right of=H\Nhidden-3] (O) {};

    \foreach \source in {1,...,4}
        \foreach \dest in {1,...,5}
            \path (H1-\dest)  edge  (I-\source);

    \foreach [remember=\N as \lastN (initially 1)] \N in {2,...,\Nhidden}
       \foreach \source in {1,...,5}
           \foreach \dest in {1,...,1}
               \path  (H\N-\dest) edge [green, thick, domain=-2:2] (H\lastN-\source) ;
    \foreach [remember=\N as \lastN (initially 1)] \N in {2,...,\Nhidden}
       \foreach \source in {1,...,5}
           \foreach \dest in {2,...,5}
               \path (H\N-\dest) edge   (H\lastN-\source);

    \foreach \source in {1,...,5}
        \path   (O)  edge (H\Nhidden-\source);


    \node[annot,left of=hl1] {Input layer};
    \node[annot,right of=hl\Nhidden] {Output layer};
\end{tikzpicture}
\caption{ An example of backpropagation of two hidden layers. Green edges represent the stochastic gradient  $G$ in Algorithm  1 .} \label{fig:M1}
\end{figure}
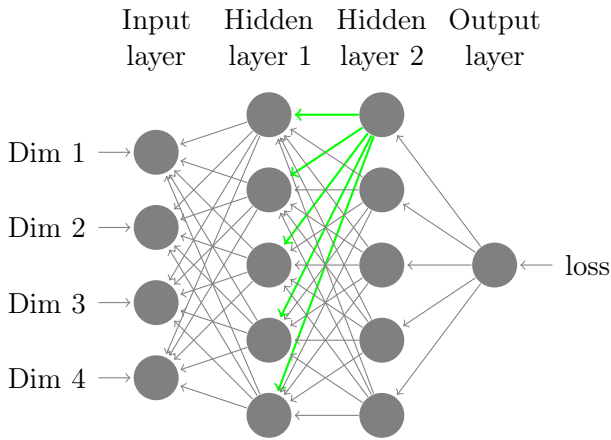

\begin{algorithm}
  \caption{Gradient Descent with Line Search Method}
  \begin{algorithmic}[1]
    \Function {line-search}{$x, b_0,\nabla F(x)$}
     \State $x_{new}\gets x-\frac{1}{b_0} \nabla F(x)$
      \While{$F(x_{new})> F(x)-\frac{b_0}{2} \| \nabla F(x)\|^2$}
        \State $b_0\gets 2b_0$
         \State $x_{new}\gets x-\frac{1}{b_0} \nabla F(x)$
      \EndWhile
      \State \textbf{return} $x_{new}$
    \EndFunction 
  \end{algorithmic}
\end{algorithm}

%
\begin{algorithm}[H]
  \caption{AdaGrad-Norm with momentum in PyTorch}
\begin{algorithmic}[1]
  \State {\bfseries Input:}   Initialize $x_0 \in \mathbb{R}^d, b_{0}>0, v_0\leftarrow 0, j \leftarrow 0, \beta \leftarrow 0.9$,  and the total iterations $N$.
    \For{\texttt{ $j = 0,1, \ldots, N$ }}
    \State Generate  $\xi_{j}$ and $G_{j} = G(x_{j}, \xi_{j})$
    \State $v_{j+1} \leftarrow \beta v_{j} + (1-\beta)G_j$ 
 	 \State   $ x_{j+1} \leftarrow x_{j} - \frac{ v_{j+1} }{b_{j+1}}$  with $b_{j+1}^2 \leftarrow  b_{j}^2 +  \|G_{j} \|^2$ 
       \EndFor
\end{algorithmic}
  \end{algorithm}


\end{document}